\documentclass[journal]{IEEEtran}
\usepackage{url}
\usepackage[cmex10]{amsmath}
\usepackage{amsxtra}
\usepackage{amsfonts}
\usepackage{amssymb}
\usepackage{amsthm}
\usepackage{graphicx}
\usepackage{epstopdf}
\usepackage[noend]{algorithmic}
\ifCLASSOPTIONcompsoc
  \usepackage[caption=false,font=normalsize,labelfont=sf,textfont=sf]{subfig}
\else
  \usepackage[caption=false,font=footnotesize]{subfig}
\fi

\usepackage{cite}
\usepackage{multirow}

\usepackage{hyperref}

\newcommand{\Dset}{\mathcal{D}}
\newcommand{\iid}{\emph{i.i.d.}}

\newcommand{\F}{{\mathcal F}}
\newcommand{\G}{{\mathcal G}}
\renewcommand{\L}{{\mathcal L}}
\newcommand{\R}{\mathbb{R}}
\newcommand{\eps}{\varepsilon}

\DeclareMathOperator*{\argmin}{argmin}

\newcommand{\citep}{\cite}
\newcommand{\citet}{\cite}

\newtheorem{theorem}{Theorem}
\newtheorem*{theorem*}{Theorem}
\theoremstyle{plain}

\newtheorem{lemma}{Lemma}

\newtheorem{proposition}{Proposition}

\begin{document} 

\title{Multi-criteria Similarity-based Anomaly Detection using Pareto Depth Analysis}

\author{Ko-Jen Hsiao, Kevin S.~Xu, \IEEEmembership{Member,~IEEE}, Jeff Calder, 
	and~Alfred O.~Hero III, \IEEEmembership{Fellow,~IEEE}
\thanks{This work was partially supported by ARO grant W911NF-12-1-0443 and W911NF-09-1-0310. The paper is submitted to IEEE TNNLS Special Issue on Learning in Non-(geo)metric Spaces for review on October 28, 2013, revised on July 26, 2015 and accepted on July 30, 2015.
A preliminary version of this work is reported in the conference publication \citep{hsiao2012}. 

K.-J.~Hsiao is with the WhisperText, 69 Windward Avenue, Venice, CA 90291 (email: coolmark@umich.edu).

K.~S.~Xu is with the Department of Electrical Engineering and Computer Science, University of Toledo, Toledo, OH 43606, USA (email: kevinxu@outlook.com).

J.~Calder is with the Department of Mathematics, University of California, Berkeley, CA 94720, USA (email: jcalder@berkeley.edu).

A.~O.~Hero III is with the Department of Electrical Engineering and Computer Science, University of Michigan, Ann Arbor, MI 48109, USA (email: hero@umich.edu).

This work was done while the first three authors were at the University of Michigan.
}}
\maketitle

\begin{abstract}

We consider the problem of identifying patterns in a data set that exhibit 
anomalous behavior, often referred to as anomaly detection. 
\emph{Similarity-based} anomaly detection algorithms detect abnormally large 
amounts of similarity or dissimilarity, e.g.~as measured by nearest neighbor 
Euclidean distances between a test sample and the training samples.  
In many application domains there may not exist a single 
dissimilarity measure that captures all possible anomalous patterns. 
In such cases, multiple dissimilarity measures can be defined, including 
non-metric measures, and one can test for anomalies by scalarizing 
using a non-negative linear combination of them.
If the relative importance of the different dissimilarity measures are not 
known in advance, as in many anomaly detection applications, the anomaly 
detection algorithm may need to be executed multiple times with 
different choices of weights in the linear combination. 
In this paper, we propose a method for similarity-based anomaly detection 
using a novel \emph{multi-criteria} dissimilarity measure, the Pareto depth. 
The proposed \emph{Pareto depth analysis} (PDA) anomaly detection algorithm  
uses the concept of Pareto optimality to detect anomalies under 
multiple criteria without having to 
run an algorithm multiple times with different choices of weights. 
The proposed PDA approach is provably better than using linear combinations 
of the criteria and shows superior performance on experiments with synthetic 
and real data sets. 

\begin{keywords}
multi-criteria dissimilarity measure, similarity-based learning,
combining dissimilarities, Pareto front, scalarization gap, partial correlation
\end{keywords}

\end{abstract}

\section{Introduction}
\label{intro}
Identifying patterns of anomalous behavior in a data set, often referred to as 
anomaly detection, is an important problem with diverse applications 
including intrusion detection in computer networks, detection of credit card 
fraud, and medical informatics \citep{Hodge2004,Chandola2009}. 
\emph{Similarity-based} approaches to anomaly detection have generated much interest 
\cite{Byers1998,Angiulli2002,Eskin2002,Breunig2000,Zhao2009,Hero2006,Sricharan2011} 
due to their relative simplicity and 
robustness as compared to model-based, cluster-based, and density-based 
approaches \citep{Hodge2004,Chandola2009}. 
These approaches typically involve the calculation of similarities or 
dissimilarities 
between data samples using a single dissimilarity criterion, such as 
Euclidean distance. 
Examples include approaches based on $k$-nearest neighbor ($k$-NN) distances 
\cite{Byers1998,Angiulli2002,Eskin2002}, local neighborhood densities 
\cite{Breunig2000}, local p-value estimates \cite{Zhao2009}, and geometric 
entropy minimization \cite{Hero2006,Sricharan2011}. 

In many application domains, such as those involving categorical data, 
it may not be possible or practical to represent 
data samples in a geometric space in order to compute Euclidean distances. 
Furthermore, \emph{multiple} dissimilarity measures 
corresponding to \emph{different criteria} 
may be required to detect certain types of anomalies. 
For example, consider the problem of detecting anomalous object trajectories 
in video sequences of different lengths. 
Multiple criteria, such as dissimilarities in object speeds or trajectory 
shapes, can be used to detect a greater range of anomalies than any single 
criterion. 

In order to perform anomaly detection using these multiple criteria, one 
could first combine the dissimilarities for each criterion using a non-negative linear combination then apply a (single-criterion) anomaly detection algorithm. 
However, in many applications, the importance of the different criteria are 
not known in advance. 
It is thus difficult to determine how much weight to assign to each criterion, so one may have to run the anomaly detection algorithm multiple times using different weights selected by a grid search or similar method. 

We propose a novel 
\emph{multi-criteria} approach for similarity-based anomaly detection using 
\emph{Pareto depth analysis} (PDA). 
PDA uses the concept of Pareto optimality, which is the typical method for defining optimality when there may 
be multiple conflicting criteria for comparing items. 
An item is said to be Pareto-optimal if there does not exist another item that 
is better or equal in \emph{all} of the criteria. 
An item that is Pareto-optimal is optimal in the usual sense 
under some (not necessarily linear) combination of the criteria. 
Hence PDA is able to detect anomalies under multiple 
combinations of the criteria without explicitly forming these 
combinations. 

\begin{figure*}[t]
\begin{center}
  \subfloat[]{\includegraphics[width=5.5 cm]{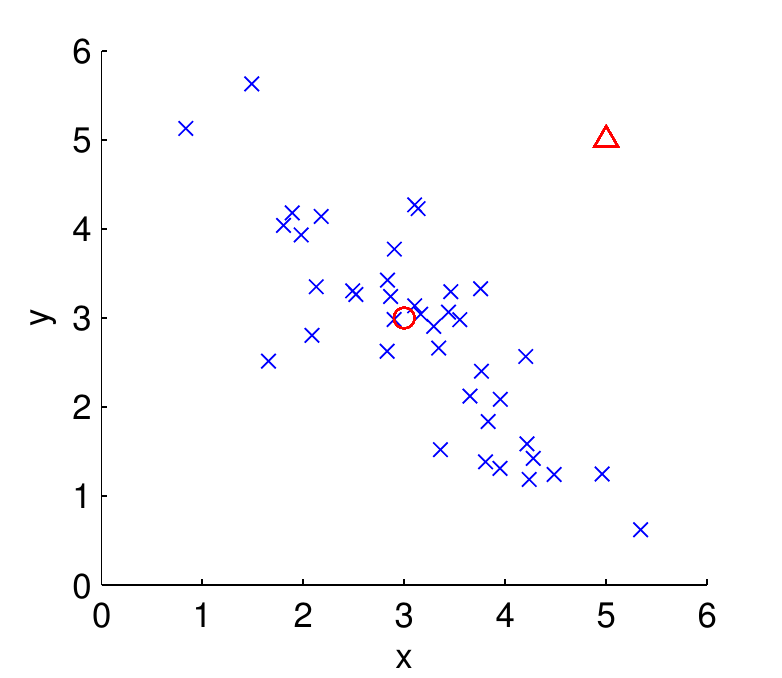}}
\quad
  \subfloat[]{\includegraphics[width=5.5 cm]{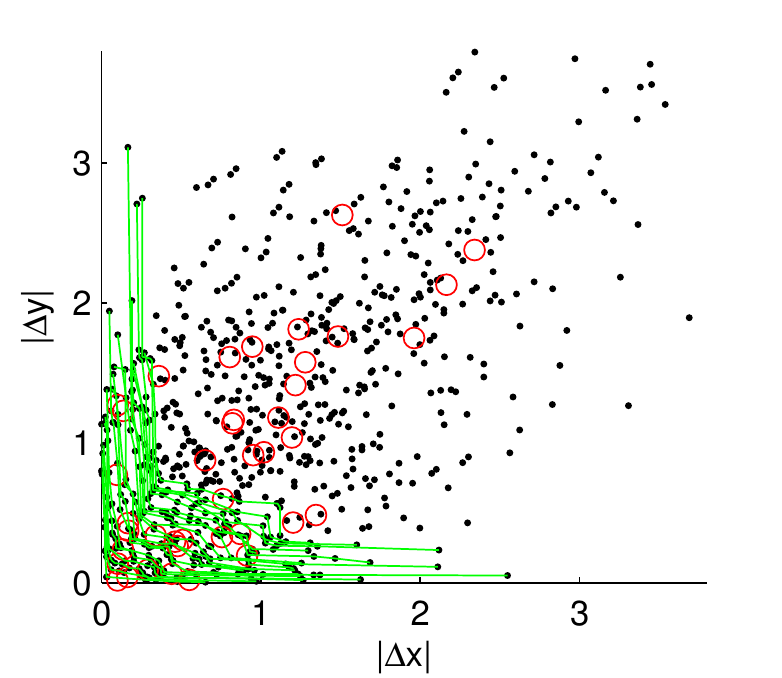}}
\quad
  \subfloat[]{\includegraphics[width=5.5 cm]{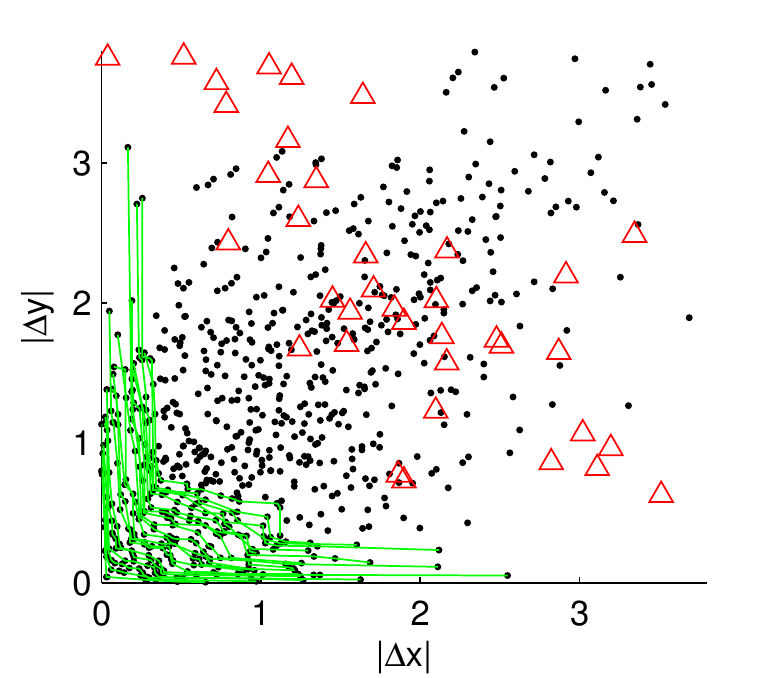}}
  \caption{(a) Illustrative example with $40$ training samples 
  (blue x's) and $2$ test samples (red circle and triangle) in $\R^2$. 
  (b) Dyads for the training samples (black dots) along with first 
  $20$  Pareto fronts (green lines) under two criteria: $|\Delta x|$ and 
  $|\Delta y|$. 
  The Pareto fronts induce a partial ordering on the set of dyads. 
  Dyads associated with the test sample marked by the red circle concentrate around 
  \emph{shallow fronts} (near the lower left of the figure). 
  (c) Dyads associated with the test sample marked by the red triangle 
  concentrate around \emph{deep fronts}.}
  \label{bi_sample}
  \end{center}
\end{figure*}

The PDA approach involves creating \emph{dyads} corresponding to 
dissimilarities 
between pairs of data samples under all of the criteria. 
Sets of Pareto-optimal dyads, called \emph{Pareto fronts}, are then computed. The first Pareto front (depth one) is the set of non-dominated dyads. The second Pareto front (depth two) is obtained by removing these non-dominated dyads, i.e.~peeling off the first front, and recomputing the first Pareto front of those remaining. This process continues until no dyads remain. In this way, each dyad is assigned to a Pareto front at some depth (see  Fig.~\ref{bi_sample} for illustration).

The \emph{Pareto depth} of a dyad is a novel measure of dissimilarity between a 
pair of data samples under multiple criteria. 
In an unsupervised anomaly detection setting, the majority of the training samples are assumed to be nominal. 
Thus a nominal test sample would likely be similar to many training samples under some criteria, so most dyads for the nominal test sample would appear in shallow Pareto fronts.
On the other hand, an anomalous test sample would likely be dissimilar to many training samples under many criteria, so most dyads for the anomalous test sample would be located in deep Pareto fronts. 
Thus computing the Pareto depths of the dyads corresponding to a test 
sample can discriminate between nominal and anomalous 
samples. 

Under the assumption that the multi-criteria dyads can be modeled as realizations from a $K$-dimensional density, we provide a mathematical analysis of properties of the first Pareto front relevant to anomaly detection.  In particular, in the \nameref{thm:Gap} we prove upper and lower bounds on the degree to which the Pareto fronts are non-convex.  For any algorithm using non-negative linear combinations of criteria, non-convexities in the Pareto fronts contribute to an artificially inflated anomaly score, resulting in an increased false positive rate.  Thus our analysis shows in a precise sense that PDA can outperform any algorithm that uses a non-negative linear combination of the criteria.
Furthermore, this theoretical prediction is experimentally validated by comparing PDA to several single-criterion similarity-based anomaly detection algorithms in two experiments involving both synthetic and 
real data sets.

The rest of this paper is organized as follows. 
We discuss related work in Section \ref{related}. 
In Section \ref{prealg} we provide an introduction to Pareto fronts and present a 
theoretical analysis of the properties of the first Pareto front. 
Section \ref{detection} relates Pareto fronts to the multi-criteria 
anomaly detection problem, which leads to the PDA anomaly detection algorithm. Finally we present three experiments in Section \ref{exp} to provide experimental support for our theoretical results and evaluate the 
performance of PDA for anomaly detection. 

\section{Related work}
\label{related}

\subsection{Multi-criteria methods for machine learning}
Several machine learning methods utilizing Pareto optimality have previously 
been proposed; an overview can be found in \citep{Jin2008}. 
These methods typically formulate supervised machine learning problems as 
multi-objective optimization problems over a potentially infinite set of candidate items where finding even the first Pareto 
front is quite difficult, often requiring multi-objective evolutionary algorithms. 
These methods differ from our use of Pareto optimality because we consider  Pareto fronts created from 
a finite set of items, so 
we do not need to employ sophisticated algorithms in order to find these 
fronts. 
Rather, we utilize Pareto fronts to form a statistical criterion for anomaly detection.

Finding the Pareto front of a finite set of items has also been referred to in 
the literature as the skyline query \citep{borzsony2001skyline,tan2001efficient} or the maximal vector problem \citep{kung1975finding}. 
Research on skyline queries has focused on how to efficiently compute or approximate items on the first Pareto front and efficiently store the results in memory.
Algorithms for skyline queries can be used in the proposed PDA approach for computing Pareto fronts.
Our work differs from skyline queries because the focus of PDA is the utilization of \emph{multiple} Pareto fronts for the purpose of multi-criteria anomaly detection, not the efficient computation or approximation of the first Pareto front. 

Hero and Fleury \citet{hero2004pareto} introduced a method for gene ranking using multiple Pareto 
fronts that is related to our approach. 
The method ranks genes, in order of interest to a 
biologist, by creating Pareto fronts on the data samples, i.e.~the genes. 
In this paper, we consider Pareto fronts of \emph{dyads}, which correspond to 
dissimilarities between \emph{pairs} of data samples under multiple criteria  rather than the samples 
themselves, and use the distribution 
of dyads in Pareto fronts to perform multi-criteria anomaly detection rather than gene ranking.  

Another related area is multi-view learning \citep{Blum1998,Sindhwani2005}, 
which involves learning from data represented by multiple sets of features, 
commonly referred to as ``views''. 
In such a case, training in one view is assumed to help to improve learning in another view. 
The problem of view disagreement, where samples take on different classes in 
different views, has recently been investigated \citep{Christoudias2008}. 
The views are similar to criteria in our problem setting. 
However, in our setting, different criteria may be orthogonal and could even 
give contradictory information; hence there may be severe view disagreement. 
Thus training in one view could actually worsen performance 
in another view, so the problem we consider differs from multi-view learning. 
A similar area is that of multiple kernel learning \citep{Gonen2011}, 
which is typically 
applied to supervised learning problems, unlike the unsupervised anomaly 
detection setting we consider.

\subsection{Anomaly detection}
\label{relAnomaly}
Many methods for anomaly detection have previously been proposed.
Hodge and Austin \citet{Hodge2004} and Chandola et al.~\citet{Chandola2009} both provide extensive surveys of 
different anomaly detection methods and applications. 

This paper focuses on the similarity-based approach to anomaly detection, also 
known as instance-based learning. 
This approach typically involves transforming similarities between a test 
sample and training samples into an anomaly score. 
Byers and Raftery \citet{Byers1998} proposed to use the distance between a sample and its 
$k$th-nearest neighbor as the anomaly score for the sample; similarly, 
Angiulli and Pizzuti \citet{Angiulli2002} and Eskin et al.~\citet{Eskin2002} proposed to the use the sum of the 
distances between a sample and its $k$ nearest neighbors. 
Breunig et al.~\citet{Breunig2000} used an anomaly score based on the local density of the 
$k$ nearest neighbors of a sample. 
Hero \citet{Hero2006} and Sricharan and Hero \citet{Sricharan2011} introduced non-parametric adaptive anomaly detection methods using geometric entropy minimization, based on random $k$-point 
minimal spanning trees and bipartite $k$-nearest neighbor ($k$-NN) graphs, respectively. 
Zhao and Saligrama \citet{Zhao2009} proposed an anomaly detection algorithm k-LPE using local 
p-value estimation (LPE) based on a $k$-NN graph. 
The aforementioned anomaly detection methods only depend on the data through the pairs of data points (dyads) that define the edges in the $k$-NN graphs.
These methods are designed for 
a single criterion, unlike the PDA anomaly detection 
algorithm that we propose in this paper, which accommodates dissimilarities corresponding to multiple criteria. 

Other related approaches for anomaly detection include $1$-class support vector 
machines (SVMs) \citep{Scholkopf2001}, where an SVM classifier is trained given 
only samples from a single class, and tree-based methods, where 
the anomaly score of a data sample is determined by its depth in a tree or 
ensemble of trees. 
Isolation forest \cite{Liu2008} and SCiForest \cite{Liu2010} are two tree-based 
approaches, targeted at detecting isolated and clustered anomalies, 
respectively, using depths of samples in an ensemble of trees. 
Such tree-based approaches utilize depths to form anomaly scores, 
similar to PDA; however, they operate on feature representations of the data 
rather than on dissimilarity representations. 
Developing multi-criteria extensions of such non-similarity-based methods is 
beyond the scope of this paper and would be worthwhile future work.

\section{Pareto fronts and their properties}
\label{prealg}

Multi-criteria optimization and Pareto fronts have been studied in many application areas in computer science, economics and the social sciences. An overview can be found in  \citep{ehrgott2005}. The proposed PDA method in this paper utilizes the notion of Pareto optimality, which we now introduce.

\subsection{Motivation for Pareto optimality}

Consider the following problem: given $n$ items, denoted by the set $\mathcal{S}$, 
and $d$ criteria for evaluating each item, denoted by functions $f_1, \ldots, 
f_d$, select $x \in \mathcal{S}$ that minimizes $[f_1(x), \ldots, f_d(x)]$. 
In most settings, it is not possible to find a single item $x$ which simultaneously minimizes $f_i(x)$ for all $i \in \{1, \ldots, d\}$. 
Many approaches to the multi-criteria optimization problem reduce to combining all of the criteria into a single criterion, a process often referred to as \emph{scalarization} \citep{ehrgott2005}. A common approach is to use a non-negative linear combination of the $f_i$'s and find the item that minimizes the linear combination. Different choices of weights in the linear combination yield different minimizers. In this case, one would need to identify a set of optimal solutions corresponding to different weights using, for example, a grid search over the weights.

A more robust and powerful approach involves identifying the set of \emph{Pareto-optimal} items. 
An item $x$ is said to \emph{strictly dominate} another item $x^*$ if 
$x$ is no greater than $x^*$ in each criterion and 
$x$ is less than $x^*$ in at least one criterion. 
This relation can be written as $x\succ x^*$ if $f_i(x) \leq f_i(x^*)$ for 
\emph{each} $i$ and $f_i(x) < f_i(x^*)$ for \emph{some} $i$. 
The set of Pareto-optimal items, called the \emph{Pareto front}, is the 
set of items in $\mathcal{S}$ that are not strictly dominated by another item 
in $\mathcal{S}$. 
It contains all of the minimizers that are found using non-negative linear combinations, 
but also includes other items that cannot be found by linear 
combinations.
Denote the Pareto front by $\F_1$, which we call the first Pareto front. 
The second Pareto front can be constructed by finding items that are 
not strictly dominated by any of the remaining items, which are members of 
the set $\mathcal{S} \setminus \F_1$. 
More generally, define the $i$th Pareto front by
$$\F_i=\text{Pareto front of the set }
\mathcal{S} \setminus \left(\bigcup_{j=1}^{i-1}\F_j\right).$$
For convenience, we say that a Pareto front $\F_i$ is \emph{deeper} 
than $\F_j$ if $i>j$. 

\subsection{Mathematical properties of Pareto fronts}
\label{theory}

The distribution of the number of points on the first Pareto front was first studied by Barndorff-Nielsen and Sobel \citet{nielsen1966}.  The problem has garnered much attention since. Bai et al.~\citet{bai2005} and Hwang and Tsai\citet{hwang2010} provide good surveys of recent results.  We will be concerned here with properties of the first Pareto front that are relevant to the PDA anomaly detection algorithm and have not yet been considered in the literature.  

Let $Y_1,\dots,Y_n$ be independent and identically distributed~(\iid)~on $\R^d$ with density function $f: \R^d \to \R$, and let $\F^n$ denote the first Pareto front of $Y_1,\dots,Y_n$.
In the general multi-criteria optimization framework, the points $Y_1,\dots,Y_n$ are the images in $\R^d$ of $n$ feasible solutions to some optimization problem under a vector of objective functions of length $d$. 
In the context of multi-criteria anomaly detection, each point $Y_i$ is a dyad corresponding to dissimilarities between two data samples under multiple criteria, and $d=K$ is the number of criteria.  

A common approach in multi-objective optimization is linear scalarization \citep{ehrgott2005}, which constructs a new single criterion as a non-negative linear combination of the $d$ criteria.  It is well-known, and easy to see, that linear scalarization will only identify Pareto-optimal points on the boundary of the convex hull of 
\[\G^n := \bigcup_{x \in \F^n} (x + \R^d_+),\]
where $\R^d_+ = \{ x \in \R^d \ \ | \ \  \forall i, \ x_i\geq 0\}$.
Although this is a common motivation for Pareto optimization methods, there are, to the best of our knowledge, no results in the literature regarding how many points on the Pareto front are missed by scalarization. We present such a result in this section, namely the \nameref{thm:Gap}.  

We define
\[\L^n = \bigcup_{\alpha \in \R^d_+} \argmin_{x \in S_n} \left\{\sum_{i=1}^d \alpha_i x_i \right\}, \ \ S_n =\{Y_1,\dots,Y_n\}. \]
The subset $\L^n \subset \F^n$ contains all Pareto-optimal points that can be obtained by some selection of of non-negative weights for linear scalarization.  Let $K_n$ denote the cardinality of $\F^n$, and let $L_n$ denote the cardinality of $\L^n$.  When $Y_1,\dots,Y_n$ are  uniformly distributed on the unit hypercube, Barndorff-Nielsen and Sobel \citet{nielsen1966} showed that
\[E(K_n) = \frac{n}{(d-1)!} \int_0^1 (1-x)^{n-1} (-\log x)^{d-1} \, dx,\]
from which one can easily obtain the asymptotics
\[E(K_n) = \frac{(\log n)^{d-1}}{(d-1)!} + O( (\log n)^{d-2}).\]
Many more recent works have studied the variance of $K_n$ and have proven central limit theorems for $K_n$.  All of these works assume that $Y_1,\dots,Y_n$ are uniformly distributed on $[0,1]^d$.    For a summary, see \citet{bai2005} and \citet{hwang2010}.  Other works have studied $K_n$ for more general distributions on domains that have smooth ``non-horizontal'' boundaries near the Pareto front \citep{baryshnikov2005} and for multivariate normal distributions on $\R^d$ \citep{ivanin1975}.  The ``non-horizontal'' condition excludes hypercubes.  

To the best of our knowledge there are no results on the asymptotics of $K_n$ for non-uniformly distributed points on the unit hypercube.  This is of great importance as it is impractical in multi-criteria optimization (or anomaly detection) to assume that the coordinates of the points are independent.  Typically the coordinates of $Y_i \in \R^d$ are the images of the \emph{same} feasible solution under several \emph{different} criteria, which will not in general be independent.
 
Here we develop results on the size of the gap between the number of items $L_n$ discoverable by scalarization compared to the number of items $K_n$ discovered on the Pareto front. The larger the gap, the more suboptimal  scalarization is relative to Pareto optimization.
Since $x \in \L^n$ if and only if $x$ is on the boundary of the convex hull of $\G^n$, the size of $\L^n$ is related to the convexity (or lack thereof) of the Pareto front.  There are several ways in which the Pareto front can be \emph{non-convex}.  

First, suppose that $Y_1,\dots,Y_n$ are distributed on some domain $\Omega \subset \R^d$ with a continuous density function $f:\overline{\Omega} \to \R$ that is strictly positive on $\overline{\Omega}$.  Let $T\subset \partial \Omega$ be a portion of the boundary of $\Omega$ such that 
\[\inf_{z\in T} \min(\nu_1(z),\dots,\nu_d(z)) > 0,\]
and
\[\{y \in \overline{\Omega} \, : \, \forall i \ \ y_i \leq x_i\} = \{x\}, \ \ {\rm for \ all} \ \ x \in T,\]
where $\nu:\partial \Omega \to \R^d$ is the unit inward normal to $\partial \Omega$.  
  The conditions on $T$ guarantee that a portion of the first Pareto front will concentrate near $T$ as $n\to \infty$.  If we suppose that $T$ is contained in the interior of the convex hull of $\Omega$, then points on the portion of the Pareto front near $T$ cannot be obtained by linear scalarization, as they are on a non-convex portion of the front.  Such non-convexities are a direct result of the geometry of the domain $\Omega$ and are depicted in Fig.~\ref{geometry}.  In a preliminary version of this work, we studied the expectation of the number of points on the Pareto front within a neighborhood of $T$ (Theorem 1 in \citep{hsiao2012}).  As a result, we showed that 
\[E(K_n - L_n) \geq \gamma n^{\frac{d-1}{d}} + O(n^\frac{d-2}{d}), \]
as $n\to \infty$, where $\gamma$ is a positive constant given by
\[ \gamma = \frac{1}{d}(d!)^{\frac{1}{d}}\Gamma\left(\frac{1}{d}\right)\int_T f(z)^{\frac{d-1}{d}} (\nu_1(z)\cdots \nu_d(z))^{\frac{1}{d}}  dz.\]
It has recently come to our attention that a stronger result was proven previously by Baryshnikov and Yukich \citet{baryshnikov2005} in an unpublished manuscript.

\begin{figure}[t]
\centering
\subfloat[]{
  \includegraphics[height=0.18\textheight]{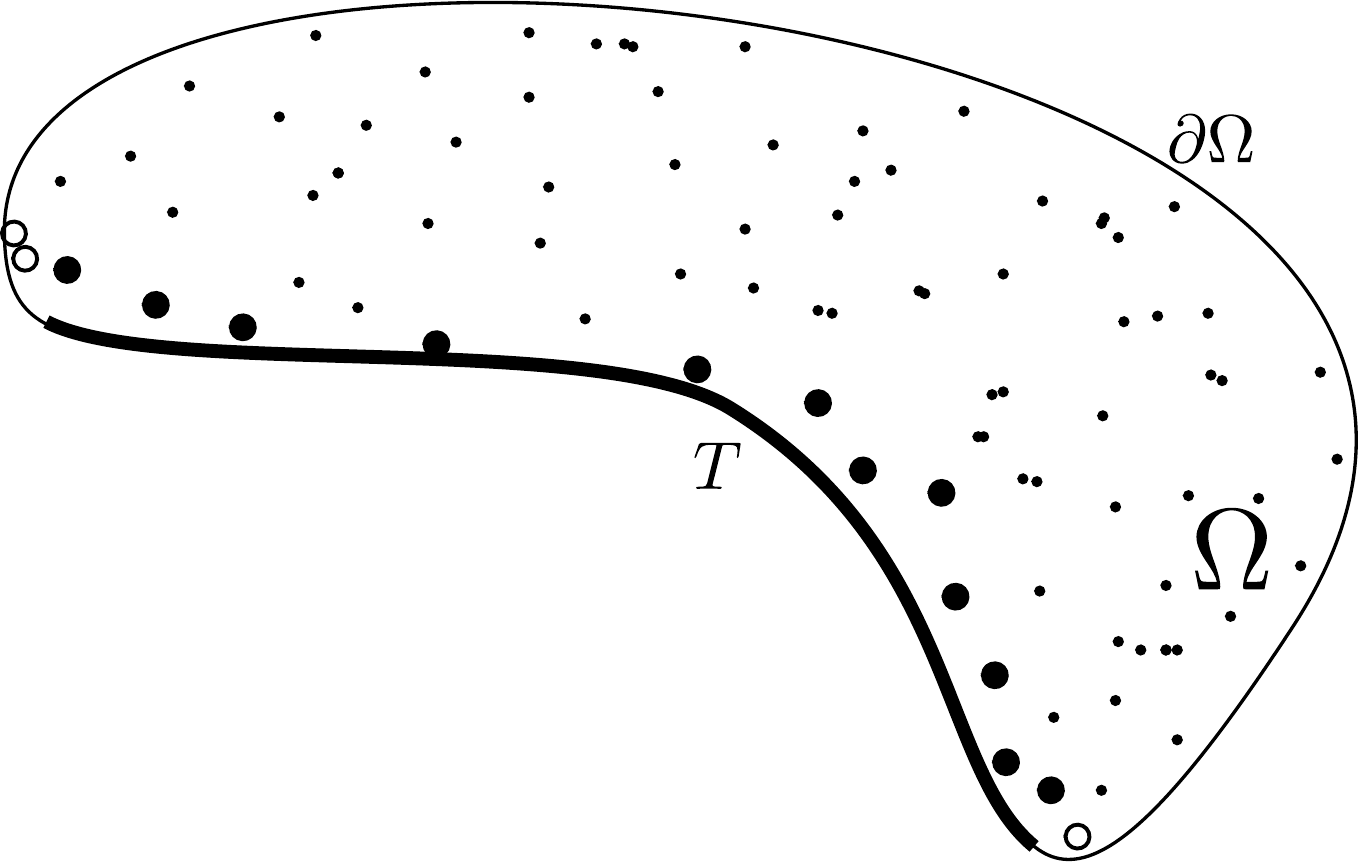} 
\label{geometry}
}
\quad
\subfloat[]{
  \includegraphics[trim = 25 17 25 25, clip=true,height=0.18\textheight]{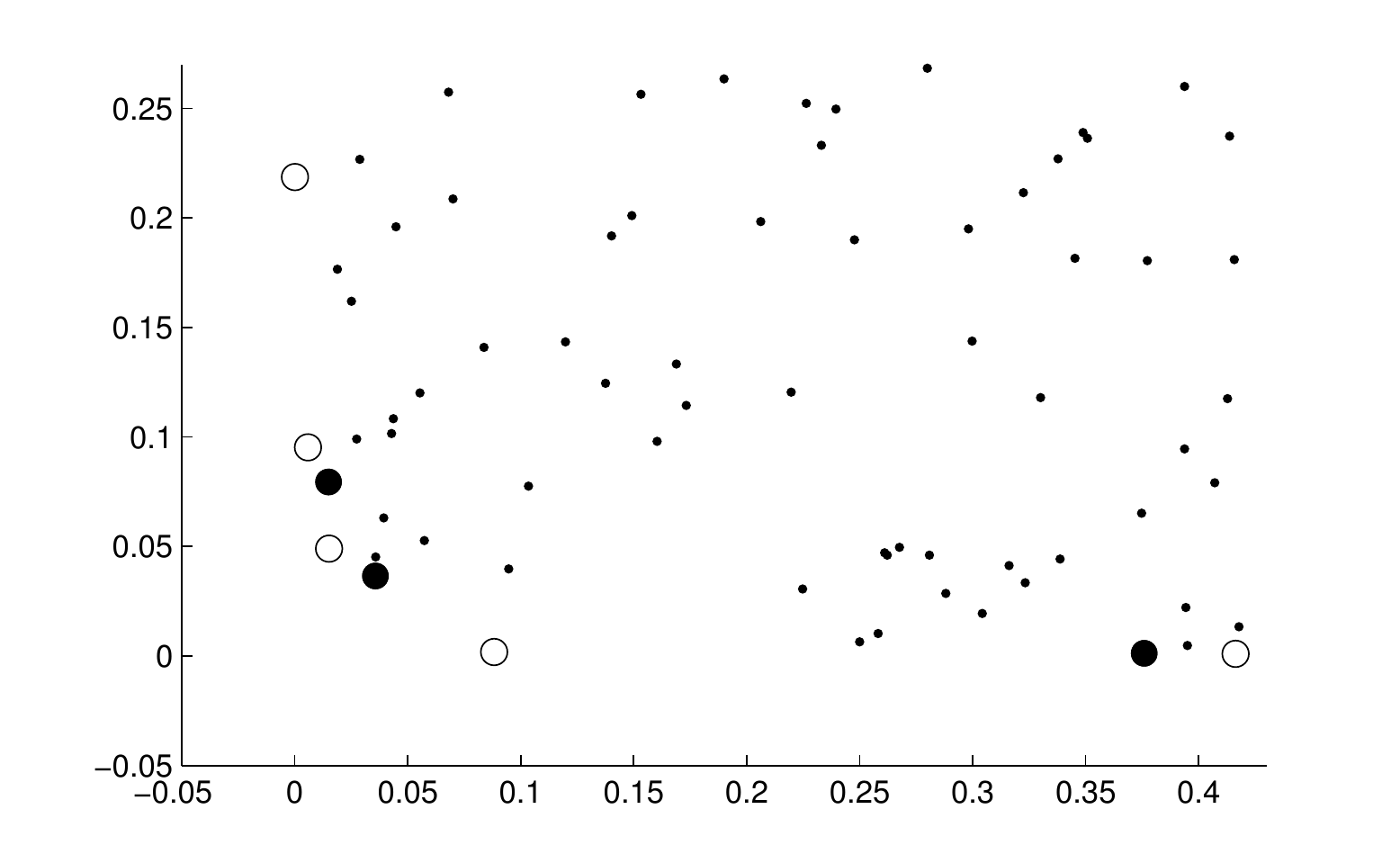}
\label{randomness}
}
  \caption[]{\subref{geometry} Non-convexities in the Pareto front induced by the geometry of the domain $\Omega$. \subref{randomness} Non-convexities due to randomness in the points. In each case, the larger points are Pareto-optimal, and the large black points \emph{cannot} be obtained by scalarization.  }
  \label{fig:theory}
\end{figure}

In practice, it is unlikely that one would have enough information about $f$ or $\Omega$ to compute the constant $\gamma$.  In this paper, we instead study a second type of non-convexity in the Pareto front.  These non-convexities are strictly due to randomness in the positions of the points and occur even when the domain $\Omega$ is convex (see Fig.~\ref{randomness} for a depiction of such non-convexities).  In the following, we assume that $Y_1,\dots,Y_n$ are \iid~on the unit hypercube $[0,1]^d$ with a bounded density function $f:[0,1]^d \to \R^d$ which is continuous at the origin and strictly positive on $[0,1]^d$.   Under these assumptions on $f$, it turns out that the asymptotics of $E(K_n)$ and $E(L_n)$ are independent of $f$.  Hence our results are applicable to a wide range of problems without the need to know detailed information about the density $f$.

Our first result provides asymptotics on $K_n$, the size of the first Pareto front. 
\begin{theorem}\label{thm:pareto-asymp}
Assume $f:[0,1]^d \to [\sigma,M]$ is continuous at the origin, and $0 < \sigma < M < \infty$.  Then 
\begin{equation*}\label{eq:arb-density-asymp}
E(K_n) \thicksim c_{n,d}:=\frac{(\log n)^{d-1}}{(d-1)!} {\rm \ as \ } n \to \infty.
\end{equation*}
\end{theorem}
The proof of Theorem \ref{thm:pareto-asymp} is provided in the Appendix.
Our second result concerns $E(L_n)$.  We are not able to get the exact asymptotics of $E(L_n)$, so we provide upper and lower asymptotic bounds.
\begin{theorem}\label{thm:small}
Assume $f:[0,1]^d \to [\sigma,M]$ is continuous at the origin, and $0 < \sigma < M < \infty$.  Then 
\begin{equation*}\label{eq:Ln-bounds}
\textstyle\frac{d!}{d^d}c_{n,d} +o((\log n)^{d-1})\leq  E(L_n)\leq \frac{3d-1}{4d-2}c_{n,d}+o((\log n)^{d-1})
\end{equation*}
as $n\to \infty$.
\end{theorem}
Theorem \ref{thm:small} provides a significant generalization of a previous result (Theorem 2 in \citep{hsiao2012}) that holds only for uniform distributions in $d=2$.  
The proof of Theorem \ref{thm:small} is also provided in the Appendix.

Combining Theorems \ref{thm:pareto-asymp} and \ref{thm:small}, we arrive at our main result:
\begin{theorem*}[Scalarization Gap Theorem]
\label{thm:Gap}
Assume $f:[0,1]^d \to [\sigma,M]$ is continuous at the origin, and $0 < \sigma < M < \infty$.  Then 
\begin{align*}\label{eq:opp-bound}
&\textstyle\frac{d-1}{4d-2} c_{n,d} + o((\log n)^{d-1}) \\
&\textstyle\hspace{1cm}\leq E(K_n - L_n) \leq \left(1 - \frac{d!}{d^d}\right) c_{n,d} + o((\log n)^{d-1}),
\end{align*}
as $n\to \infty$.
\end{theorem*}

The \nameref{thm:Gap} shows that the fraction of Pareto-optimal points that \emph{cannot} be obtained by linear scalarization is at least $\frac{d-1}{4d-2}$.  
We provide experimental evidence supporting these bounds in Section \ref{exp_sup}.   

\section{Multi-criteria anomaly detection}
\label{detection}

We now formally define the multi-criteria anomaly detection problem. 
A list of notation is provided in Table \ref{tab:Notations} for reference. 
Assume that a training set $\mathcal{X}_N=\{X_1,\ldots,X_N\}$ of 
unlabeled data samples is available. 
Given a test sample $X$, the objective of anomaly detection is to declare $X$ to 
be an anomaly if $X$ is significantly different from samples in $\mathcal{X}_N$. 
Suppose that $K>1$ different evaluation criteria are given. 
Each criterion is associated with a measure for computing 
dissimilarities.  
Denote the dissimilarity between $X_i$ and $X_j$ computed using the dissimilarity measure 
corresponding to the $l$th criterion by $d_l(i,j)$. 
Note that $d_l(i,j)$ need not be a metric; in particular it is not necessary that $d_l(i,j)$ be a distance function over the sample space or that $d_l(i,j)$ satisfy the triangle inequality. 

\begin{table}[t]
\renewcommand{\arraystretch}{1.1}
\centering
\caption{List of notation}
\label{tab:Notations}

\begin{tabular}{lp{2.7in}}
\hline
Symbol & Definition \\
\hline
$K$ & Number of criteria (dissimilarity measures) \\
$N$ & Number of training samples \\
$X_i$ & $i$th training sample \\
$X$ & Single test sample \\
$d_l(i,j)$ & Dissimilarity between training samples $X_i$ and $X_j$ using 
	$l$th criterion \\
$D_{ij}$ & Dyad between training samples $X_i$ and $X_j$ \\
$\mathcal{D}$ & Set of all dyads between training samples \\
$\mathcal{F}_i$ & Pareto front $i$ of dyads between training samples \\
$M$ & Total number of Pareto fronts on dyads between training samples \\
$D_{i}$ & Dyad between training sample $X_i$ and test sample $X$ \\
$e_i$ & Pareto depth of dyad $D_i$ between training sample $X_i$ and
	test sample $X$ \\
$k_l$ & Number of nearest neighbors in criterion $l$ \\
$s$ & Total number of nearest neighbors (over all criteria) \\
$v(X)$ & Anomaly score of test sample $X$ \\
\hline
\end{tabular}
\end{table}

We define a \emph{dyad} between a pair of samples $i \in \{1,\ldots,N\}$ and $j \in \{1,\ldots,N\} \setminus i$ by a vector $D_{ij}=[d_1(i,j), \dots, d_K(i,j)]^T \in 
\mathbb{R}_{+}^K$. 
There are in total ${N\choose 2}$ different dyads for the training set.  
For convenience, denote the set of all dyads by $\Dset$. 
By the definition of strict dominance in Section \ref{prealg}, a dyad 
$D_{ij}$ 
strictly dominates another dyad $D_{i^* j^*}$ if $d_l(i,j) \leq d_l(i^*,j^*)$ 
for all $l \in \{1, 
\ldots, K\}$ and $d_l(i,j) < d_l(i^*,j^*)$ for some $l$. 
The first Pareto front $\F_1$ corresponds to the set of dyads from $\Dset$ that 
are not strictly dominated by any other dyads from $\Dset$. 
The second Pareto front $\F_2$ corresponds to the set of dyads from 
$\Dset \setminus \F_1$ that are not strictly dominated by any other dyads from 
$\Dset \setminus \F_1$, and so on, as defined in Section \ref{prealg}. 
Recall that we refer to $\F_i$ as a \emph{deeper} front than $\F_j$ if $i>j$.

\subsection{Pareto fronts on dyads}
\label{PDA_anomaly}

For each training sample $X_i$, there are $N-1$ dyads corresponding to its 
connections 
with the other $N-1$ training samples. 
If most of these dyads are located at shallow Pareto fronts, then the 
dissimilarities between $X_i$ and the other $N-1$ training samples are small under 
\emph{some} combination of the criteria. 
Thus, $X_i$ is likely to be a nominal sample. 
This is the basic idea of the proposed multi-criteria anomaly detection 
method using PDA. 

We construct Pareto fronts $\F_1,\ldots,\F_M$ of the dyads from the training 
set, where the total number of fronts $M$ 
is the required number of fronts such that each dyad is a member of a front. 
When a test sample $X$ is obtained, we create new dyads corresponding to 
connections between $X$ and training samples, as illustrated in Fig.~\ref{bi_sample}. 
Like with many other similarity-based anomaly detection methods, 
we connect each test sample to its $k$ nearest neighbors. 
$k$ could be different for each criterion, so we denote $k_l$ as the choice 
of $k$ for criterion $l$. 
We create $s = \sum_{l=1}^K k_l$ new dyads $D_1, D_2, \ldots, D_s$, corresponding to 
the connections between $X$ and the union of the $k_l$ nearest neighbors in each 
criterion $l$. 
In other words, we create a dyad between test sample $X$ and training sample $X_i$ if $X_i$ is among the 
$k_l$ nearest neighbors\footnote{If $X_i$ is one of the $k_l$ nearest neighbors of $X$
in multiple criteria, then multiple copies of the dyad $D_i$ are created. We have also experimented with creating only a single copy of the dyad and found very little difference in detection accuracy.} of $X$ in any criterion $l$. 
We say that $D_i$ is \emph{below} a front $\F_j$ if 
$D_i \succ D \text{ for some } D \in \F_j$,
i.e.~$D_i$ strictly dominates at least a single dyad in $\F_j$. 
Define the \emph{Pareto depth} of $D_i$ by 
\begin{equation*}
e_i=\min\{j\,|\,D_i \text{ is below } \F_j\}.
\end{equation*}
Therefore if $e_i$ is large, then $D_i$ will be near deep fronts, and the distance between $X$ and $X_i$ will be  
large under \emph{all} combinations of the $K$ criteria. 
If $e_i$ is small, then $D_i$ will be near shallow fronts, so the 
distance between $X$ and $X_i$ will be small under 
\emph{some} combination of the $K$ criteria. 

\subsection{Anomaly detection using Pareto depths}
\label{score}
In $k$-NN based anomaly detection algorithms such as those mentioned in 
Section \ref{relAnomaly}, the \emph{anomaly score} is a function of the $k$ 
nearest neighbors to a test sample. 
With multiple criteria, one could define an anomaly score by 
scalarization. 
From the probabilistic properties of Pareto fronts discussed in Section \ref{theory}, we know that Pareto optimization methods identify more Pareto-optimal points than linear scalarization methods and significantly more Pareto-optimal points than a single weight for scalarization\footnote{Theorems \ref{thm:pareto-asymp} and \ref{thm:small} require \iid~samples, but dyads are not independent. However, there are $O(N^2)$ dyads, and each dyad is only dependent on $O(N)$ other dyads.  This suggests that the theorems should also hold for the non-\iid~dyads as well, and it is supported by experimental results presented in Section \ref{exp_sup}.}. 

This motivates us to develop a \emph{multi-criteria anomaly score} using Pareto fronts.
We start with the observation from Fig.~\ref{bi_sample} that dyads 
corresponding to a nominal test sample are typically located near shallower
fronts than dyads corresponding to an anomalous test sample. 
Each test sample is associated with $s = \sum_{l=1}^K k_l$ new dyads, where the $i$th dyad 
$D_i$ has depth $e_i$. 
The Pareto depth $e_i$ is a \emph{multi-criteria dissimilarity measure} that indicates the dissimilarity between the test sample and training sample $i$ under multiple combinations of the criteria. 
For each test sample $X$, we define the anomaly score $v(X)$ to be the mean of the $e_i$'s, which corresponds to the average depth of the $s$ dyads associated with $X$, or equivalently, the average of the multi-criteria dissimilarities between the test sample and its $s$ nearest neighbors. Thus the anomaly score can be easily computed and compared to a decision 
threshold $\rho$ using the test
$$
v(X) = \frac{1}{s} \sum_{i=1}^s e_i \overset{H_1}{\underset{H_0}{\gtrless}} \rho.
$$

Recall that the \nameref{thm:Gap} provides bounds on the fraction of dyads on the first Pareto front that \emph{cannot} be obtained by linear scalarization. 
Specifically, at least $\frac{K-1}{4K-2}$ dyads will be missed by linear scalarization on average. 
These dyads will be associated with deeper fronts by linear scalarization, which will artificially inflate the anomaly score for the test sample, resulting in an increased false positive rate for any algorithm that utilizes non-negative linear combinations of criteria. 
This effect then cascades to dyads in deeper Pareto fronts, which also get assigned inflated anomaly scores. 
We provide some evidence of this effect on a real data experiment in 
Section \ref{traj}. 
Finally, the lower bound increases monotonically in $K$, which implies that the PDA approach gains additional advantages over linear combinations as the number of criteria increases. 

\subsection{PDA anomaly detection algorithm}
\label{sec:Algorithm}
\begin{figure}[t]
Training phase:
\begin{algorithmic}[1]
	\FOR {$l = 1 \to K$}
		\STATE {Calculate pairwise dissimilarities $d_l(i,j)$ between all 
		training samples $X_i$ and $X_j$}
	\ENDFOR
	\STATE {Create dyads $D_{ij} = [d_1(i,j), \ldots, d_K(i,j)]$ for all 
	training samples}
	\STATE {Construct Pareto fronts on set of all dyads until each dyad is 
	in a front}
\end{algorithmic}
Testing phase:
\begin{algorithmic}[1]
	\STATE {$nb \leftarrow [\,]$} \COMMENT{empty list}
	\FOR {$l = 1 \to K$}
		\STATE {Calculate dissimilarities between test sample $X$ 
		and all training samples in criterion $l$}
		\STATE {$nb_l \leftarrow k_l$ nearest neighbors of $X$}
		\STATE {$nb \leftarrow [nb,nb_l]$} \COMMENT{append neighbors to list}
	\ENDFOR
	\STATE {Create $s = \sum_{l=1}^K k_l$ new dyads $D_i$ between $X$ and training 
	samples in $nb$}
	\FOR {$i = 1 \to s$}
		\STATE {Calculate depth $e_i$ of $D_i$}
	\ENDFOR
	\STATE {Declare $X$ an anomaly if $v(X) = (1/s) \sum_{i=1}^s e_i > \rho$}	
\end{algorithmic}
\caption{Pseudocode for PDA anomaly detection algorithm.}
\label{alg}
\end{figure}

Pseudocode for the PDA anomaly detector is shown in Fig.~\ref{alg}.
The training phase involves creating $N \choose 2$ dyads 
corresponding to all pairs of training samples. 
Computing all pairwise dissimilarities in each criterion requires 
$O(mKN^2)$ floating-point operations (flops), where $m$ denotes the 
number of dimensions involved in computing a dissimilarity. 
The time complexity of the training phase is dominated by the construction of 
the Pareto fronts by non-dominated sorting. 
Non-dominated sorting is used heavily by the evolutionary computing community; 
to the best of our knowledge, the fastest algorithm for non-dominated sorting 
was proposed by Jensen \cite{jensen2003} and later generalized by 
Fortin et al.~\cite{Fortin2013} and utilizes $O(N^2 \log^{K-1} (N^2))$ comparisons. 
The complexity analyses in \cite{jensen2003,Fortin2013} are asymptotic in $N$ 
and assume $K$ fixed. 
We are unaware of any analyses of its asymptotics in $K$. 
Another non-dominated sorting algorithm proposed by 
Deb et al.~\cite{deb2002} constructs all of the Pareto fronts 
using $O(KN^4)$ comparisons, which is linear in the number of 
criteria $K$ but scales poorly with the number of training samples $N$. 
We evaluate how both approaches scale with $K$ and $N$ experimentally in Section \ref{sec:ComputTime}.

The testing phase 
involves creating dyads between the test sample and the $k_l$ nearest 
training samples in criterion $l$, which requires $O(mKN)$ flops. 
For each dyad $D_i$, we need to calculate the depth $e_i$. 
This involves comparing the test dyad with training dyads on 
multiple fronts until we find a training dyad that is dominated by the 
test dyad. 
$e_i$ is the front that this training dyad is a part of. 
Using a binary search to select the front and another binary search 
to select the training dyads within the front to compare to, 
we need to make $O(K\log^2 N)$ comparisons (in the worst case) to 
compute $e_i$. 
The anomaly score is computed by taking the mean of the $s$ $e_i$'s 
corresponding to the test sample; the score is then compared against a 
threshold $\rho$ to determine whether the sample is anomalous. 

\subsection{Selection of parameters}
\label{param}
The parameters to be selected in PDA are $k_1, \ldots, k_K$, which 
denote the number of nearest neighbors in each criterion. 

The selection of such parameters in unsupervised learning problems is very difficult in general. 
For each criterion $l$, we construct a $k_l$-NN graph using the corresponding 
dissimilarity measure. 
We construct symmetric $k_l$-NN graphs, i.e.~we connect samples $i$ and $j$ if 
$i$ is one of the $k_l$ nearest neighbors of $j$ or $j$ is one of the $k_l$ 
nearest neighbors of $i$. 
We choose $k_l = \log N$ as a starting point and, if necessary, increase $k_l$ until the $k_l$-NN graph is connected. 
This method of choosing $k_l$ is motivated by asymptotic results for connectivity in $k$-NN graphs and has been used as a 
heuristic in other unsupervised learning problems, such as spectral clustering 
\citep{vonLuxburg2007}. 
We find that this heuristic works well in practice, including on a real data set of pedestrian trajectories, which we present in Section \ref{traj}.

\section{Experiments}
\label{exp}
We first present an experiment involving the scalarization gap for dyads (rather than \iid~samples). Then we compare the PDA method with five single-criterion anomaly detection algorithms on a simulated data set and a real data set\footnote{The code for the experiments is available at \url{http://tbayes.eecs.umich.edu/coolmark/pda}.}. 
The five algorithms we use for comparison are as follows:
\begin{itemize}
\item kNN: distance to the $k$th nearest neighbor \citep{Byers1998}.
\item kNN sum: sum of the distances to the $k$ nearest neighbors \citep{Angiulli2002,Eskin2002}.
\item k-LPE: localized p-value estimate using the $k$ nearest neighbors \citep{Zhao2009}.
\item LOF: local density of the $k$ nearest neighbors \citep{Breunig2000}.
\item 1-SVM: $1$-class support vector machine \citep{Scholkopf2001}.
\end{itemize}

For these methods, we use linear combinations of the criteria with 
different weights (linear scalarization) to compare performance with the proposed multi-criteria PDA method. 
We find that the accuracies of the nearest neighbor-based methods do not vary 
much in our experiments for $k = 3, \ldots, 10$. 
The results we report use $k=6$ neighbors. 
For the $1$-class SVM, it is difficult to choose a bandwidth for the Gaussian 
kernel without having labeled anomalous samples. 
Linear kernels have been found to perform similarly to Gaussian kernels 
on dissimilarity representations for SVMs in classification tasks 
\cite{Chen2009}; 
hence we use a linear kernel on the scalarized dissimilarities for the 
$1$-class SVM.

\subsection{Scalarization gap for dyads}
\label{exp_sup}

\begin{figure}[t]
  \centering
\includegraphics[clip=true,trim = 0 25 40 30, width=0.4\textwidth]{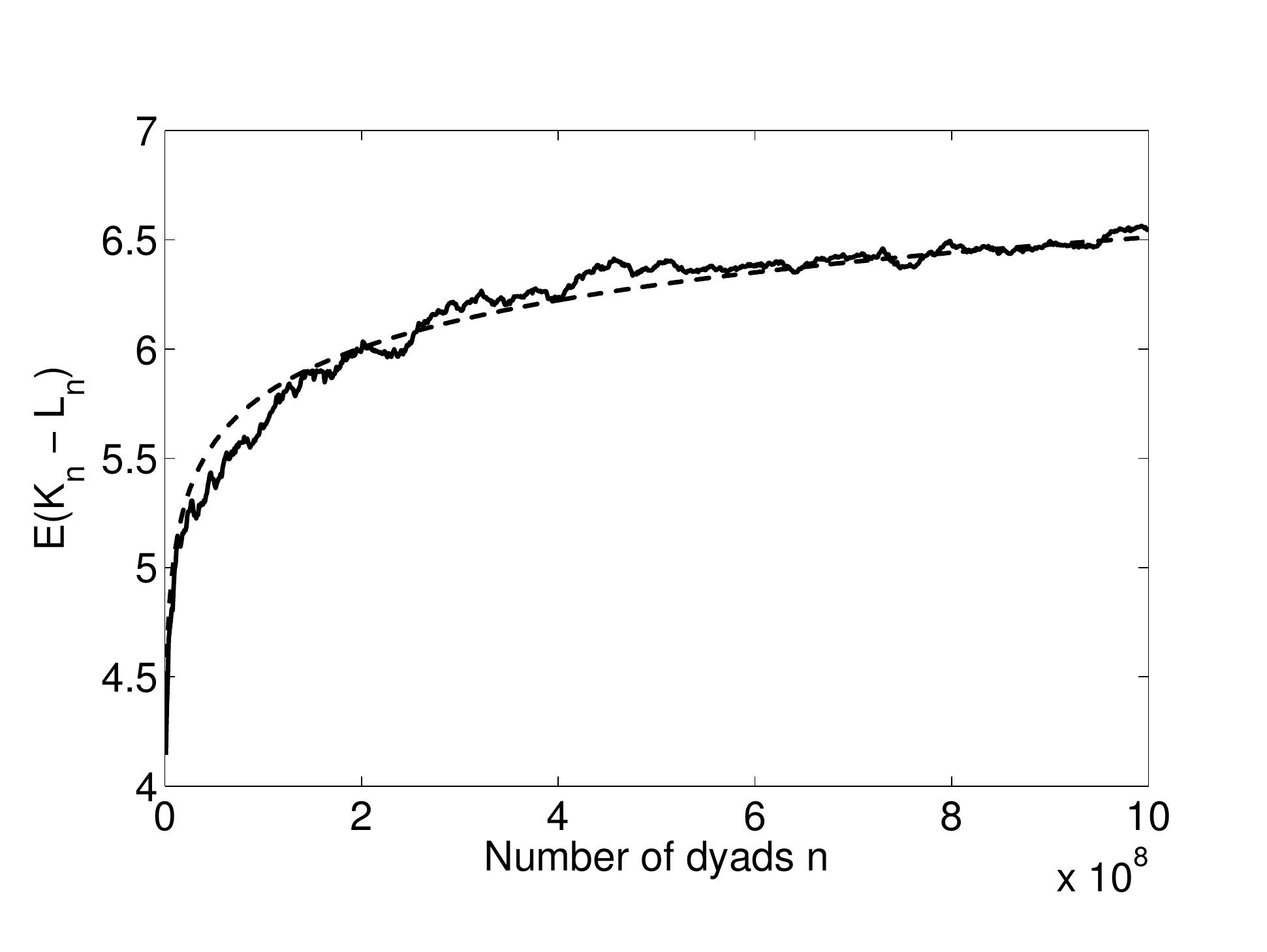}
\caption{ Sample means for $K_n - L_n$ versus number of dyads for dimension $d=2$.  Note the expected logarithmic growth.  The dotted line indicates the curve of best fit $y=0.314 \log n$.} 
\label{dyad_data}
\end{figure}

\begin{figure}[t]
  \centering
\includegraphics[clip=true,trim = 0 33 40 30, width=0.4\textwidth]{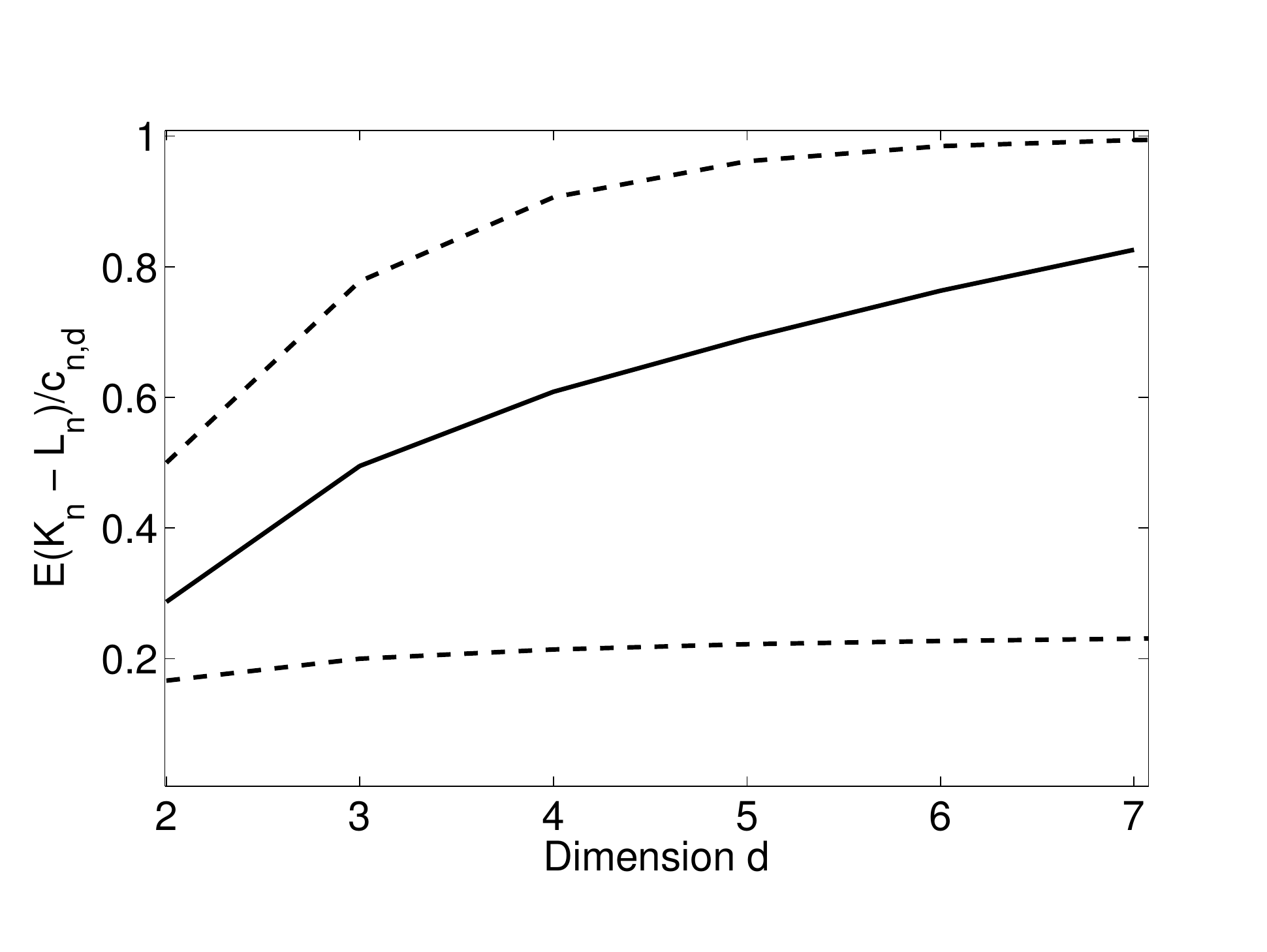}
\caption{Sample means for $(K_n -L_n)/c_{n,d}$ versus dimension for $n=100,\!128$ dyads. The upper and lower bounds established in the \nameref{thm:Gap} are given by the dotted lines in the figure.  We see in the figure that the fraction of Pareto optimal points that are not obtainable by linear scalarization increases with dimension.  }
\label{vsD-fig}
\end{figure}

Independence of $Y_1,\dots,Y_n$ is built into the assumptions of Theorems 
\ref{thm:pareto-asymp} and \ref{thm:small}, and thus, the \nameref{thm:Gap}, but it is clear that dyads (as constructed in Section \ref{detection}) are not independent.  Each dyad $D_{ij}$ represents a connection between two independent training samples $X_i$ and $X_j$.  For a given dyad $D_{ij}$, there are $2(N-2)$ corresponding dyads involving $X_i$ or $X_j$, and these are clearly not independent from $D_{ij}$.  However, all other dyads are independent from $D_{ij}$.  So while there are $O(N^2)$ dyads, each dyad is independent from all other dyads except for a set of size $O(N)$.  Since the \nameref{thm:Gap} is an asymptotic result, the above observation suggests it should hold for the dyads even though they are not~\emph{i.i.d}.  In this subsection we present some experimental results which suggest that the \nameref{thm:Gap} does indeed hold for dyads. 

We generate synthetic dyads here by drawing \emph{i.i.d.}~uniform samples in $[0,1]^2$ and then constructing dyads corresponding to the two criteria $|\Delta x|$ and $|\Delta y|$, which denote the absolute differences between the $x$ and $y$ coordinates, respectively.  The domain of the resulting dyads is again the box $[0,1]^2$.  In this case, the \nameref{thm:Gap} suggests that $E(K_n - L_n)$ should grow logarithmically.  Fig.~\ref{dyad_data} shows the sample means of $K_n-L_n$ versus number of dyads and a best fit logarithmic curve of the form $y = \alpha \log n$, where $n = {N \choose 2}$ denotes the number of dyads.  We vary the number of dyads between $10^6$ to $10^9$ in increments of $10^6$ and compute the size of $K_n-L_n$ after each increment. We compute the sample means over $1,\!000$ realizations.   A linear regression on $y/\log n$ versus $\log n$ gives $\alpha =0.314$, which falls in the range specified by the \nameref{thm:Gap}.

We next explore the dependence of $K_n-L_n$ on the dimension $d$.  Here, we generate $100,\!128$ dyads (corresponding to $N=448$ points in $[0,1]^d$) in the same way as before, for dimensions $d=2,\dots,7$.  The criteria in this case correspond to the absolute differences in each dimension.    In Fig.~\ref{vsD-fig} we plot $E(K_n - L_n)/c_{n,d}$ versus dimension to show the fraction of Pareto-optimal points that \emph{cannot} be obtained by scalarization. Recall from Theorem \ref{thm:pareto-asymp} that
\[E(K_n) \thicksim c_{n,d} = \frac{(\log n)^{d-1}}{(d-1)!} \ \ \text{ as }  n \to\infty.\]
 Based on the figure, one might conjecture that the fraction of unattainable Pareto optimal points converges to $1$ as $d \to \infty$.  If this is true, it would essentially imply that linear scalarization is useless for identifying dyads on the first Pareto front when there are a large number of criteria. As before, we compute the sample means over $1,\!000$ realizations of the experiment.

\subsection{Simulated experiment with categorical attributes}
\label{dimExp_categorical}
In this experiment, we perform multi-criteria anomaly detection on simulated 
data with multiple groups of categorical attributes. 
These groups could represent different types of attributes. 
Each data sample consists of $K$ groups of $20$ categorical attributes. 
Let $A_{ij}$ denote the $j$th attribute in group $i$, and let $n_{ij}$ 
denote the number of possible values for this attribute. 
We randomly select between $6$ and $10$ possible values for each attribute with equal probability independent of all other attributes. 
Each attribute is a random variable described by a categorical 
distribution, where the parameters $q_1, \dots, q_{n_{ij}}$ of the categorical 
distribution are sampled from a Dirichlet distribution with parameters 
$\alpha_1, \dots, \alpha_{n_{ij}}$. 
For a nominal data sample, we set $\alpha_1=5$ and $\alpha_2, \dots, 
\alpha_{n_{ij}}=1$ for each attribute $j$ in each group $i$. 

To simulate an anomalous data sample, we randomly select a group $i$ with 
probability $p_i$ for which the parameters of the Dirichlet distribution 
are changed to $\alpha_1 = \dots = \alpha_{n_{ij}} = 1$ for each 
attribute $j$ in group $i$. 
Note that different anomalous samples may differ in the group that is 
selected. 
The $p_i$'s are chosen such that $p_i / p_j = i/j$ with  
$\sum_{i=1}^K p_i = 0.5$, so that the probability that a test sample is 
anomalous is $0.5$. 
The non-uniform distribution on the $p_i$'s results in some criteria being 
more useful than others for identifying anomalies. 
The $K$ criteria for anomaly detection are taken to be the dissimilarities 
between data samples for each of the $K$ groups of attributes. 
For each group, we calculate the dissimilarity over the attributes using 
a dissimilarity measure for anomaly detection on categorical data 
proposed in \citet{Eskin2002}\footnote{We obtain similar results with 
several other dissimilarity measures for categorical data, including 
the Goodall2 and IOF measures described in the survey paper by 
Boriah et al.~\citet{boriah2008similarity}.}.

We draw $400$ training samples from the nominal distribution and $400$ test 
samples from a mixture of the nominal and anomalous distributions. 
For the single-criterion algorithms, which we use as baselines for comparison, we use linear scalarization with multiple choices 
of weights. 
Since a grid search scales exponentially with the number of criteria $K$ and is computationally 
intractable even for moderate values of $K$, we instead uniformly sample 
$100K$ weights from the ($K-1$)-dimensional simplex. 
In other words, we sample $100K$ weights from a uniform distribution over all 
convex combinations of $K$ criteria. 

\begin{figure}[tp]
\begin{center}
  \includegraphics[width=6.5 cm]{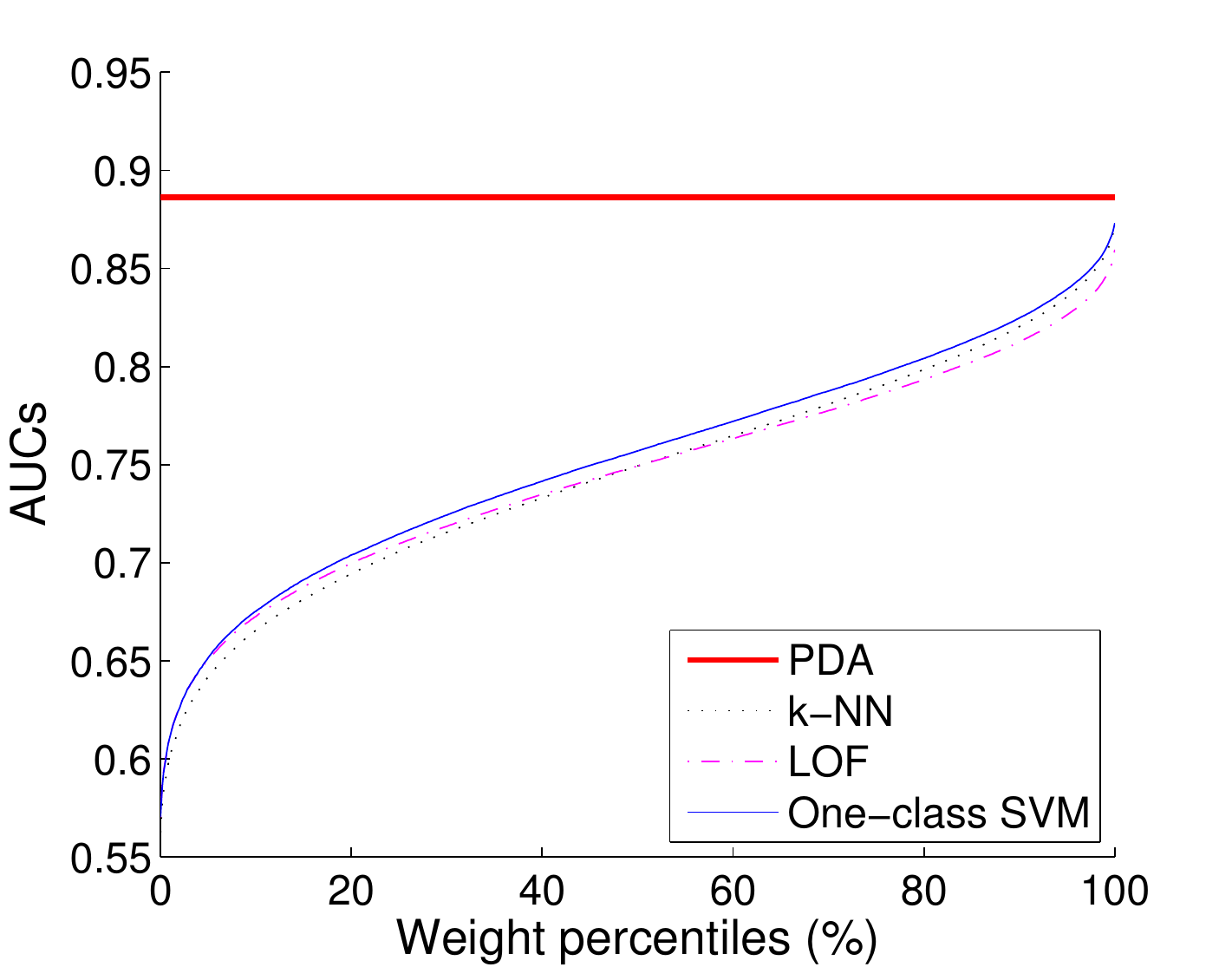}
  \caption{AUC of PDA compared to AUCs of single-criterion methods for simulated experiment. The single-criterion methods use $600$ randomly sampled weights for linear scalarization, with weights ordered from worst choice of weights (left) to best choice (right) in terms of maximizing AUC. The proposed PDA algorithm is a multi-criteria algorithm that does not require selecting weights. PDA outperforms all of the single-criterion methods, even for the best choice of weights, which is not known in advance.
 }
  \label{CategoricalAUCPercentile}
  \end{center}
\end{figure}

\begin{table}[t]
\renewcommand{\arraystretch}{1.1}
\caption{Comparison of AUCs for simulated experiment.
Best performer up to one standard error is shown in \textbf{bold}. 
PDA does not use weights so it has a single AUC. 
Median and best AUCs over all choices of weights are 
shown for the other methods.} 
\label{auc_table}
\centering
\begin{tabular}{ccccc}
\hline
\multirow{2}{*}{Method} & \multicolumn{2}{c}{AUC by weight}\\
& Median & Best \\
\hline
PDA & \multicolumn{2}{c}{\bf{0.885 $\pm$ 0.002}} \\
k-NN       &0.749 $\pm$ 0.002&    0.872 $\pm$ 0.002 \\
k-NN sum   &0.747 $\pm$ 0.002&    0.870 $\pm$ 0.002\\
k-LPE     &0.744 $\pm$ 0.002&    0.867 $\pm$ 0.002 \\
LOF       &0.749 $\pm$ 0.002&    0.859 $\pm$ 0.002 \\
1-SVM     &0.757 $\pm$ 0.002&    0.873 $\pm$ 0.002 \\
\hline
\end{tabular}
\end{table}

\subsubsection{Detection accuracy}
The different methods are evaluated using the receiver operating 
characteristic (ROC) curve and the area under the ROC curve (AUC). 
We first fix the number of criteria $K$ to be $6$. 
The mean AUCs over $100$ simulation runs
are shown in Fig.~\ref{CategoricalAUCPercentile}. 
Multiple choices of weights are used for linear scalarization for the single-criterion algorithms; 
the results are ordered from worst to best weight in terms of maximizing AUC. 
kNN, kNN sum, and k-LPE perform roughly equally so only kNN is shown in the figure. 
Table \ref{auc_table} presents a comparison of the AUC for PDA with the median and best AUCs over all choices of weights for scalarization. 
Both the mean and standard error of the AUCs over the $100$ simulation runs are shown. 
Notice that PDA outperforms even the best weighted combination for each of 
the five single-criterion algorithms and significantly outperforms the combination resulting in the median AUC, which is more representative of the performance one expects to obtain by arbitrarily choosing weights.

\begin{figure}[tp]
\begin{center}
  \includegraphics[width=6.5 cm]{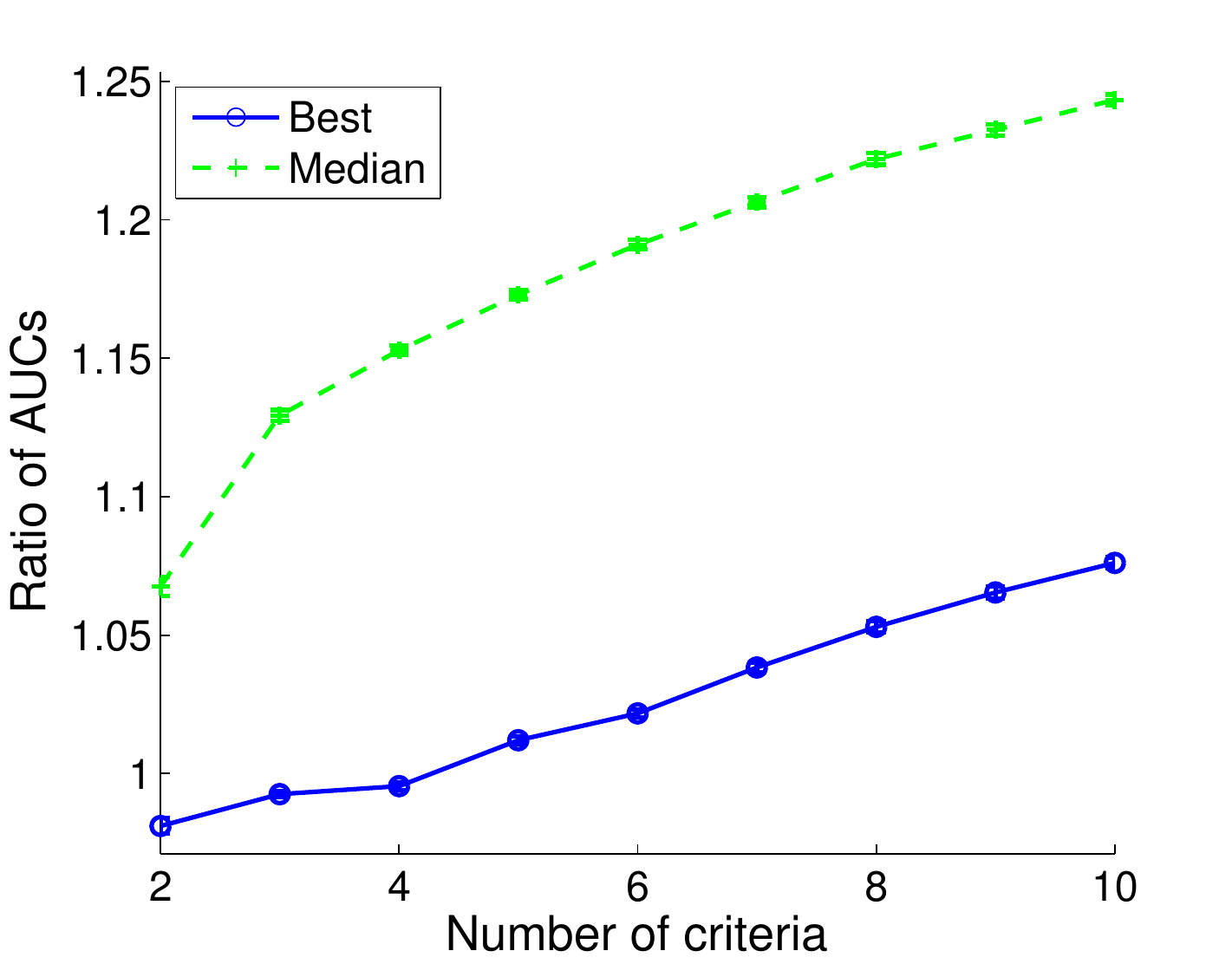}
  \caption{The ratio of the AUC for PDA compared to the best and median AUCs of scalarization using LOF as the number of criteria $K$ is varied in the simulated experiment. $100K$ choices of weights uniformly sampled from the $(K-1)$-dimensional simplex are chosen for scalarization. PDA perfoms significantly better than the median over all weights for all $K$. For $K>4$, PDA outperforms the best weights for scalarization, and the margin increases as $K$ increases.
 }
  \label{ExpDifferetK}
  \end{center}
\end{figure}

Next we investigate the performance gap between PDA and scalarization as 
the number of criteria $K$ varies from $2$ to $10$. 
The performance of the five single-criterion algorithms is very close, so we show 
scalarization results only for LOF. 
The ratio of the AUC for PDA to the AUCs of the best and median weights for scalarization 
are shown in Fig.~\ref{ExpDifferetK}. 
PDA offers a significant improvement compared to the median over the weights for scalarization. 
For small values of $K$, PDA performs roughly equally with scalarization under the best choice of weights. 
As $K$ increases, however, PDA clearly outperforms scalarization, and the gap grows 
with $K$. 
We believe this is partially due to the inadequacy of scalarization for 
identifying Pareto fronts as described in the \nameref{thm:Gap} 
and partially due to the difficulty in selecting optimal weights for the 
criteria. 
A grid search may be able to reveal better weights for 
scalarization, but it is also computationally intractable for large $K$. 
Thus we conclude that PDA is clearly the superior approach for large $K$. 

\subsubsection{Computation time}
\label{sec:ComputTime}
We evaluate how the computation time of PDA scales with varying $K$ and $N$ 
using both the non-dominated sorting procedures of Fortin et 
al.~\cite{Fortin2013} (denoted by PDA-Fortin) and Deb et al.~\cite{deb2002} 
(denoted by PDA-Deb) discussed in Section 
\ref{sec:Algorithm}. 
The time complexity of the testing phase is negligible compared to the training 
phase so we measure the computation time required to train the PDA anomaly 
detector. 

We first fix $K=2$ and measure computation time for $N$ uniformly distributed 
on a log scale from $100$ to $10,000$. 
Since the actual computation time depends heavily on the implementation of the 
non-dominated sorts, we normalize computation times by the time required to 
train the anomaly detector for $N=100$ so we can observe the scaling in $N$.

\begin{figure}[t]
\centering
\includegraphics[width=3.2in]{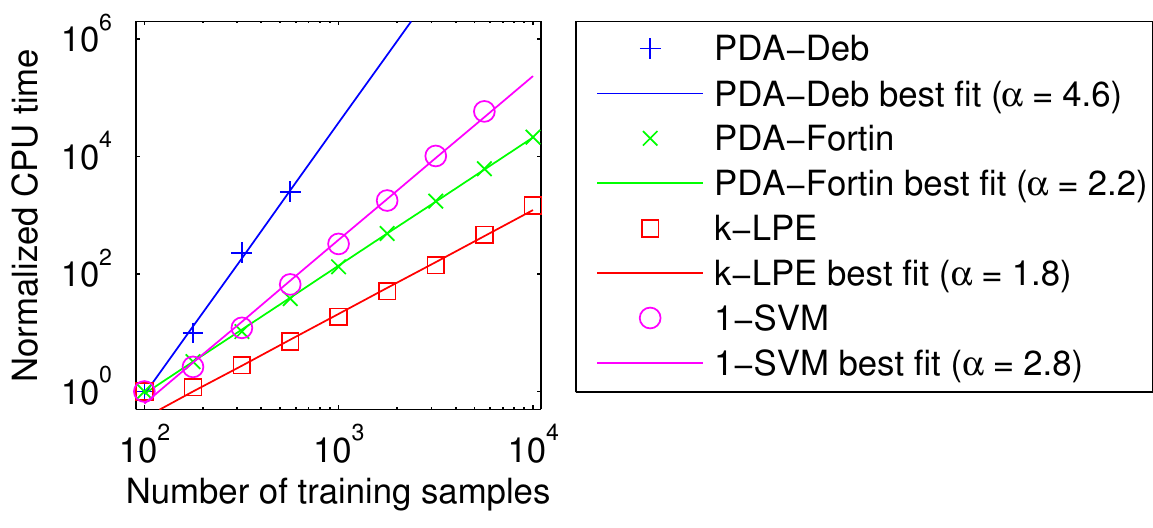}
\caption{Normalized computation time as a function of the number of training 
samples $N$ in the simulated experiment. 
Best fit curves are of the form $y = \beta N^\alpha$. 
The best fit curve for 1-SVM is extrapolated beyond $5,\!624$ samples, and the 
best fit curve for PDA-Deb is extrapolated beyond $563$ samples.}
\label{fig:CpuTimeN}
\end{figure}

The normalized times for PDA as well as k-LPE and 1-SVM are shown in Fig.~\ref{fig:CpuTimeN}. 
Best fit curves of the form $y = \beta N^\alpha$ are also plotted, with 
$\alpha$ and $\beta$ estimated by linear regression. 
PDA-Deb has time complexity of $O(KN^4)$, and the estimated exponent 
$\alpha = 4.6$. 
Of the four algorithms, it has the worst scaling in $N$. 
PDA-Fortin has time complexity of $O(N^2 \log^{K-1}(N^2))$, and the estimated 
exponent $\alpha = 2.2$, confirming that it scales much better than PDA-Deb 
and is applicable to large data sets. 
k-LPE is representative of the k-NN algorithms and has time complexity of 
$O(N^2)$; the estimated exponent $\alpha = 1.8$. 
It is difficult to determine the time complexity of 1-SVM due to its 
iterative nature. 
The estimated exponent for 1-SVM is $\alpha = 2.8$, suggesting that it scales 
worse than PDA-Fortin.

\begin{figure}[t]
\centering
\subfloat[PDA-Deb]{\label{fig:DebK}
\includegraphics[width=1.5in]{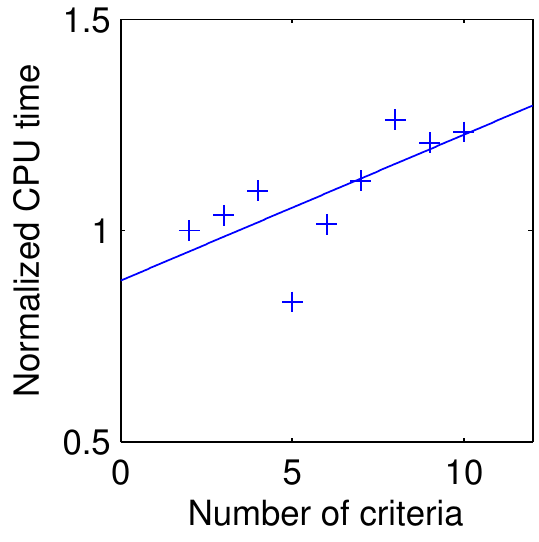}}
\qquad
\subfloat[PDA-Fortin]{\label{fig:FortinK}
\includegraphics[width=1.5in]{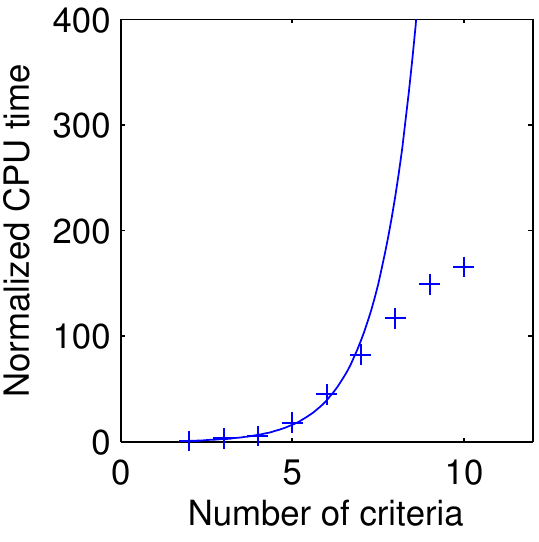}}

\caption[]{Normalized computation time as a function of the number of criteria $K$ in the simulated experiment. 
PDA-Deb \subref{fig:DebK} appears to be linear in $K$ as predicted. 
PDA-Fortin \subref{fig:FortinK} initially appears to be exponential in $K$ but the computation time does not continue to increase exponentially beyond $K=7$.}
\end{figure}

Next we fix $N=400$ and measure computation time for $K$ varying from $2$ to 
$10$. 
We normalize by the time required to train the anomaly detector for $K=2$ to 
observe the scaling in $K$. 
The normalized time for PDA-Deb is shown in Fig.~\ref{fig:DebK} along with a best fit 
line of the form $y = \alpha K$. 
The normalized time does indeed appear to be linear in $K$ and grows slowly. 
The normalized time for PDA-Fortin is shown in Fig.~\ref{fig:FortinK} along with a 
best fit 
curve of the form $y = \alpha \beta^K$ fit to $K=2, \ldots, 7$. 
Notice that the computation time initially increases exponentially but 
increases at a much slower, possibly even linear, rate beyond $K=7$. 
The analyses of time complexity from \cite{jensen2003,Fortin2013} 
are asymptotic in $N$ and assume $K$ fixed; we 
are not aware of any analyses of time complexity asymptotic in $K$. 
Our experiments suggest that PDA-Fortin is computationally tractable for 
non-dominated sorting in PDA even for large $K$. 
Finally we note that the scaling in $K$ for scalarization methods is trivial, 
depending simply on the number of choices for scalarization weights, which is 
exponential for a grid search. 

\subsection{Pedestrian trajectories}
\label{traj}
We now present an experiment on a real data set that contains thousands of 
pedestrians' trajectories in an open area monitored by a video camera \citep{Majecka2009}. 
We represent a trajectory with $p$ time samples by
$$T = \begin{bmatrix}
x_1 & x_2 & \ldots & x_p \\
y_1 & y_2 & \ldots & y_p
\end{bmatrix},$$
where $[x_t,y_t]$ denote a pedestrian's position at time step $t$. 
The pedestrian trajectories are of different lengths so we cannot simply treat the trajectories as vectors in $\mathbb{R}^p$ and calculate Euclidean distances between them. Instead, we propose to calculate dissimilarities between trajectories using two separate criteria for which trajectories may be dissimilar.

\begin{figure}[t]
\begin{center}
\subfloat[]{
  \includegraphics[width=4 cm, trim=1cm 0cm 1cm 0.7cm, clip=true]{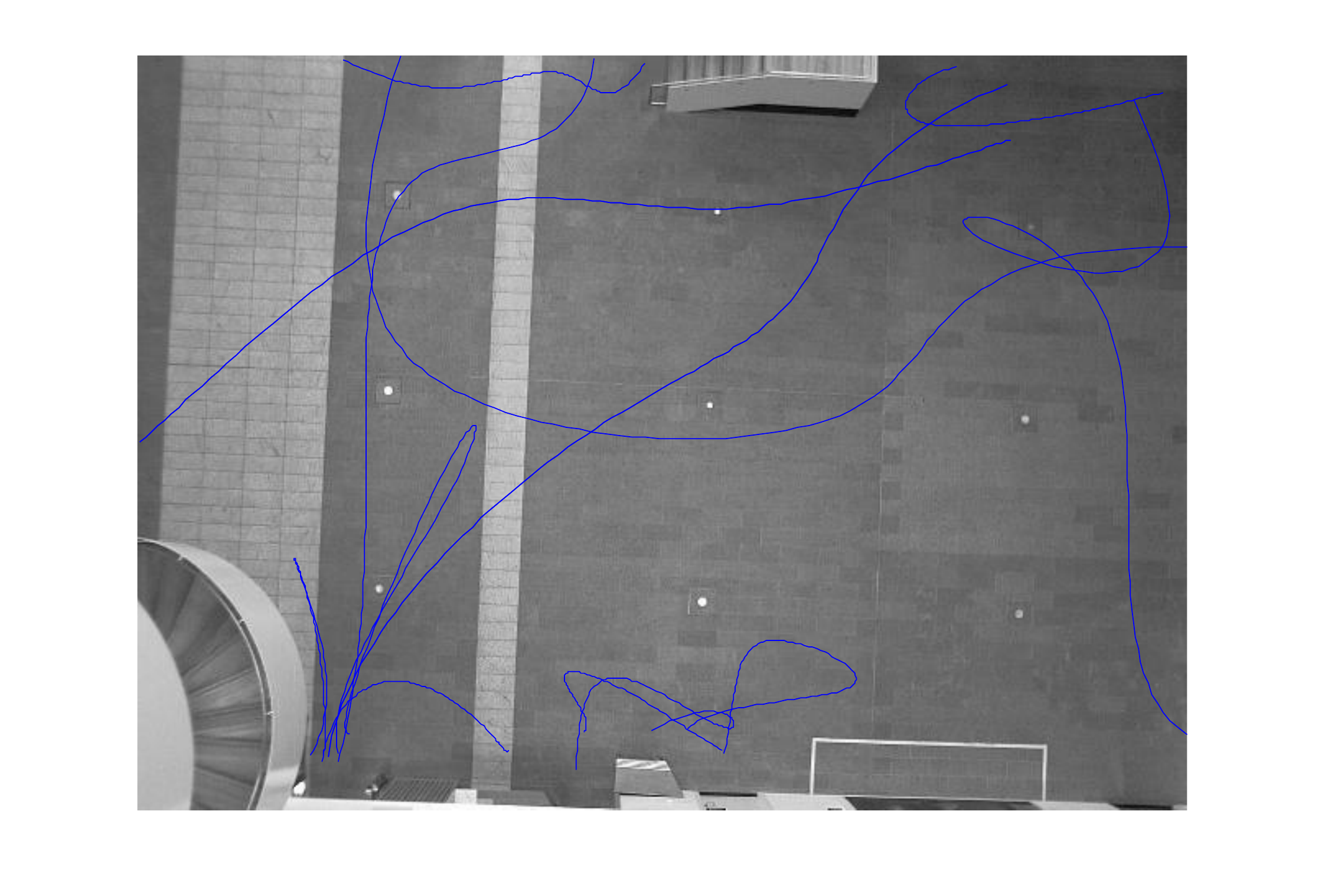}
  \label{trajAnomalous}}
  \,
\subfloat[]{
  \includegraphics[width=4 cm, trim=1cm 0cm 1cm 0.7cm, clip=true]{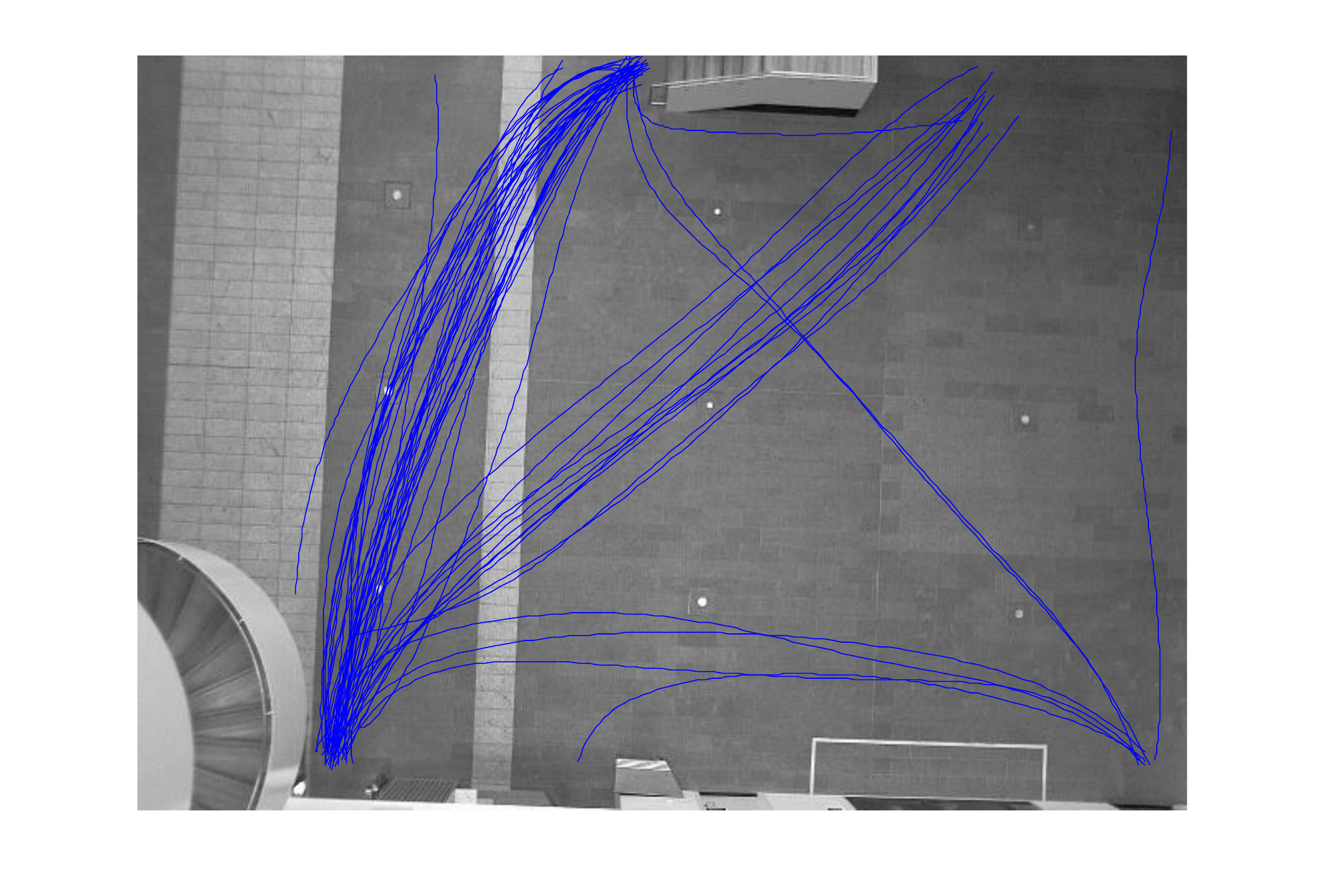}
  \label{trajNominal}}
  \caption[]{\subref{trajAnomalous} Some anomalous pedestrian trajectories detected by PDA. 
  \subref{trajNominal} Trajectories with relatively low anomaly scores. The two criteria used are walking speed and trajectory shape. Anomalous trajectories could have anomalous speeds or shapes (or both), so some anomalous trajectories may not look anomalous by shape alone. }
  \label{traj_PDA}
\end{center}
\end{figure}

The first criterion is to compute the dissimilarity in \emph{walking speed}. 
We compute the instantaneous speed at all time steps along each trajectory 
by finite differencing, i.e.~the speed of trajectory $T$ at time step $t$ 
is given by $\sqrt{(x_t - x_{t-1})^2 + (y_t - y_{t-1})^2}$. 
A histogram of speeds for each trajectory is obtained in this manner. 
We take the dissimilarity between two trajectories $S$ and $T$ to be the \emph{Kullback-Leibler (K-L) divergence} between the normalized speed histograms for those trajectories. 
K-L divergence is a commonly used measure of the difference between two probability distributions.
The K-L divergence is asymmetric; to convert it to a dissimilarity we use the symmetrized K-L divergence $D_{KL}(S||T) + D_{KL}(T||S)$ as originally defined by Kullback and Leibler \citet{kullback1951information}. 
We note that, while the symmetrized K-L divergence is a dissimilarity, it does not, in general, satisfy the triangle inequality and is not a metric.

The second criterion is to compute the dissimilarity in \emph{shape}. 
To calculate the shape dissimilarity between two trajectories, we apply a technique known as \emph{dynamic time warping} (DTW) \citep{sankoff1983time}, which first non-linearly warps the trajectories in time to match them in an optimal manner. We then take the dissimilarity to be the summed Euclidean distance between the warped trajectories. This dissimilarity also does not satisfy the triangle inequality in general and is thus not a metric. 

The training set for this experiment consists 
of $500$ randomly sampled trajectories from the data set, a small fraction 
of which may be anomalous. 
The test set consists of $200$ trajectories 
($150$ nominal and $50$ anomalous). 
The trajectories in the test set are labeled as nominal or anomalous by 
a human viewer. 
These labels are used as ground truth to evaluate anomaly detection performance. 
Fig.~\ref{traj_PDA} shows some anomalous trajectories and nominal trajectories detected using PDA.

\begin{figure}[tp]
\begin{center}
  \includegraphics[width=6.5 cm]{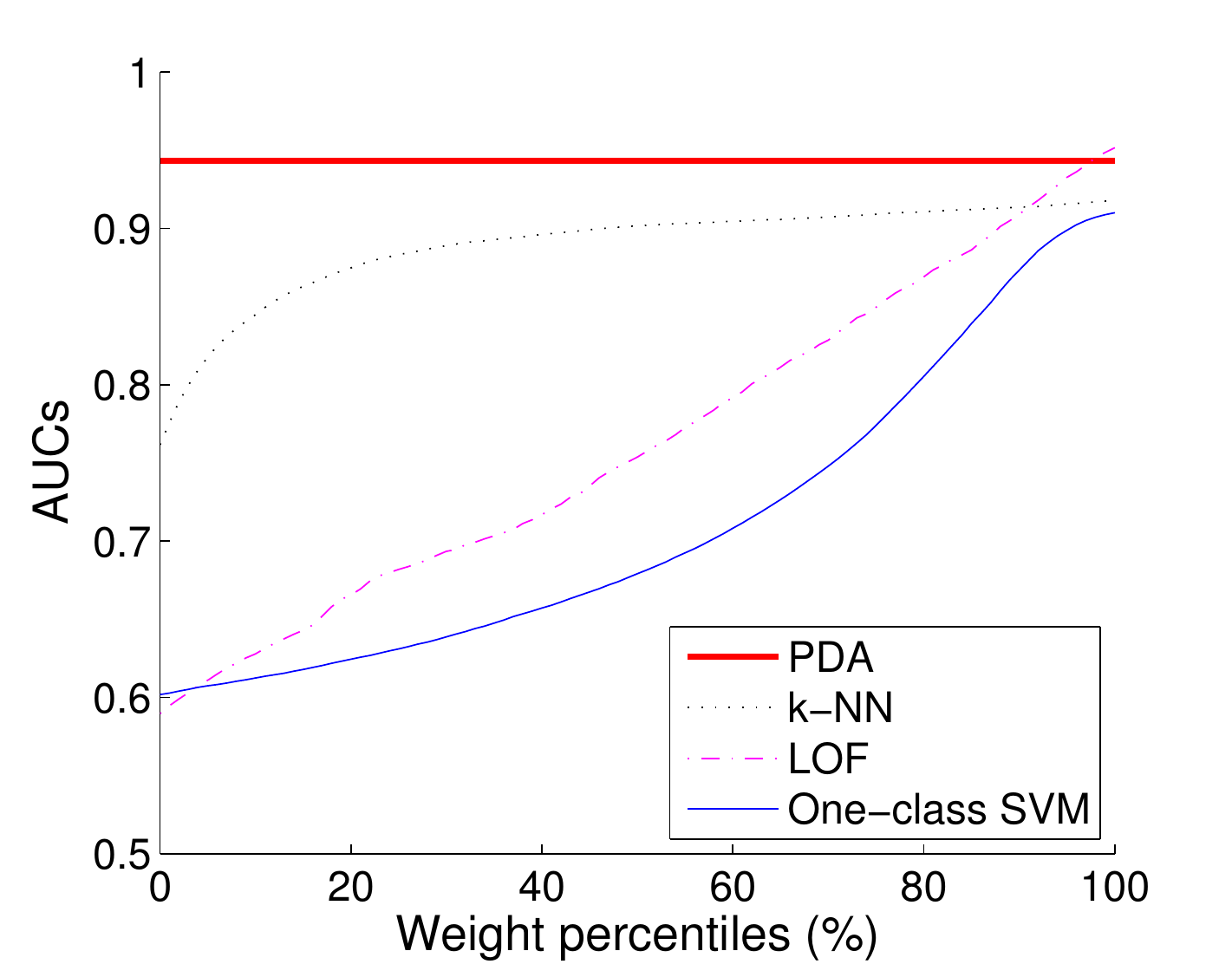}
  \caption{AUC of PDA compared to AUCs of single-criterion methods for the pedestrian trajectories experiment. The single-criterion methods use linear scalarization with 100 uniformly spaced weights; weights are ordered from worst (left) to best (right) in terms of maximizing AUC. PDA outperforms the single-criterion methods for almost all choices of weights.
 }
  \label{TrajAUCPercentile}
  \end{center}
\end{figure}

\begin{table}[t]
\renewcommand{\arraystretch}{1.1}
\caption{Comparison of AUCs for pedestrian trajectories experiment.
Best performer up to one standard error is shown in \textbf{bold}.}
\label{auc_table_traj}
\begin{center}
\begin{tabular}{ccc}
\hline
\multirow{2}{*}{Method} & \multicolumn{2}{c}{AUC by weight} \\
& Median & Best \\
\hline
PDA & \multicolumn{2}{c}{0.944 $\pm$ 0.002} \\
k-NN       &0.902 $\pm$ 0.002&    0.918 $\pm$ 0.002 \\
k-NN sum   &0.901 $\pm$ 0.003&    0.924 $\pm$ 0.002\\
k-LPE     &0.892 $\pm$ 0.003&    0.917 $\pm$ 0.002\\
LOF       &0.754 $\pm$ 0.011&    {\bf 0.952 $\pm$ 0.003} \\
1-SVM     &0.679 $\pm$ 0.011&    0.910 $\pm$ 0.003\\

\hline
\end{tabular}
\end{center}
\end{table}

\begin{figure}[t]
\begin{center}

  \includegraphics[width=6.5 cm]{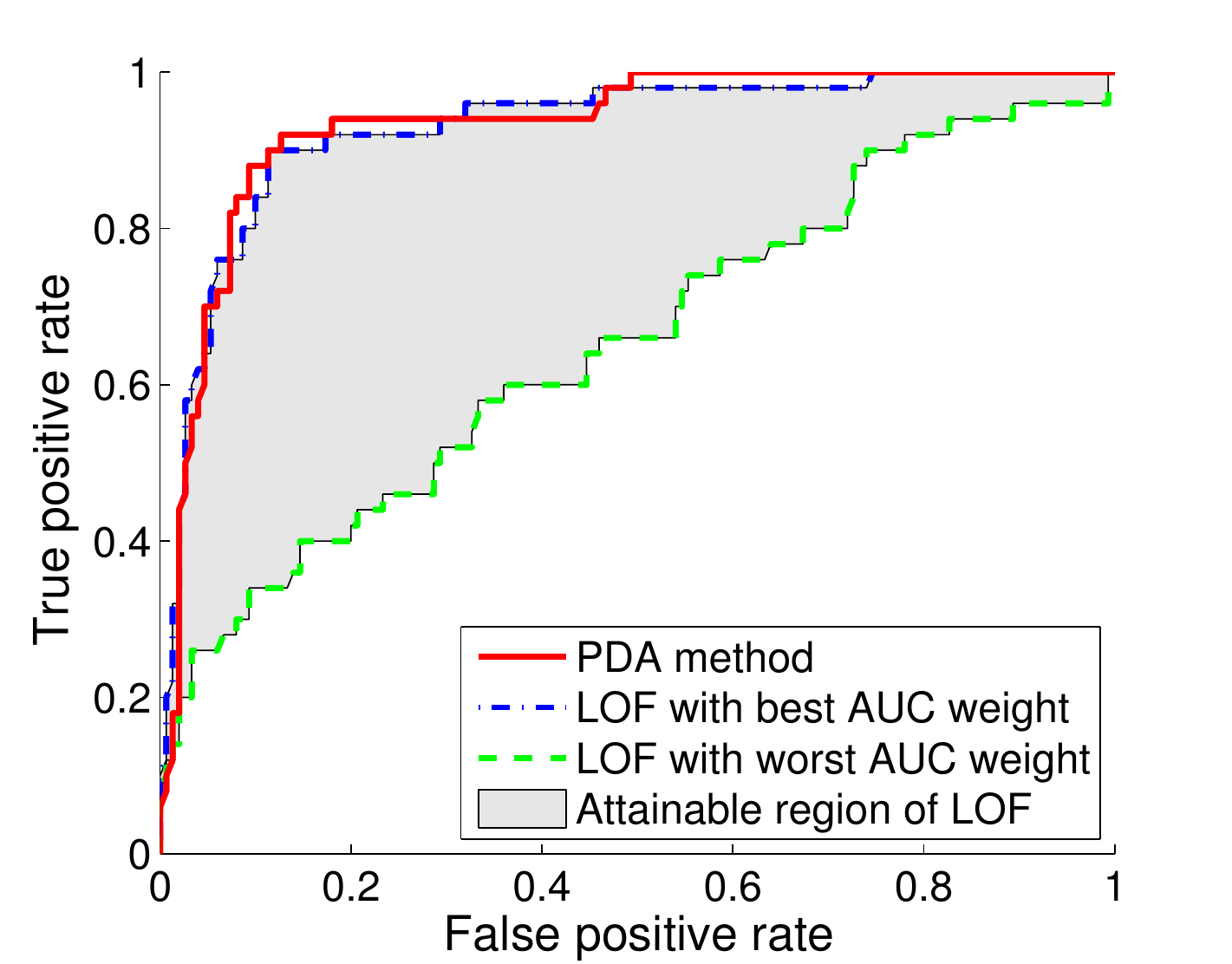}
  \caption{ROC curves for PDA and attainable region for LOF over 
  $100$ choices of weights for one run of the pedestrian trajectories experiment. The attainable region denotes the possible ROC curves for LOF corresponding to different choices of weights for linear scalarization. The ROCs for linear scalarization vary greatly as a function of the weights.}
  \label{roc_curves}
\end{center}
\end{figure}

We run the experiment $20$ times; for each run, we use a different random sample of training trajectories.
Fig.~\ref{TrajAUCPercentile} shows the performance of PDA as compared to the other 
algorithms using $100$ uniformly spaced weights for convex combinations. 
Notice that PDA has higher AUC than the other methods for almost all choices of 
weights for the two criteria. 
The AUC for PDA is shown in Table \ref{auc_table_traj} along with AUCs for the median and best choices of weights for the single-criterion methods. 
The mean and standard error over the $20$ runs are shown.
For the best choice of weights, LOF is the single-criterion method with the highest AUC, but it also has the lowest AUC for the worst choice of weights. 
For a more detailed comparison, the ROC curve for PDA and the 
attainable region for LOF (the region between the ROC curves corresponding 
to weights resulting in the best and worst AUCs) is shown in Fig.~\ref{roc_curves}.  
Note that the ROC curve for LOF can vary significantly based on the choice of weights. 
The ROC for 1-SVM also depends heavily on the weights. 
In the unsupervised setting, it is unlikely that one would be able to achieve the ROC curve corresponding to the weight with the highest AUC, so the expected performance should be closer to the median AUCs in Table \ref{auc_table_traj}.

\begin{figure}[t]
\begin{center}
  \includegraphics[width=6.5 cm]{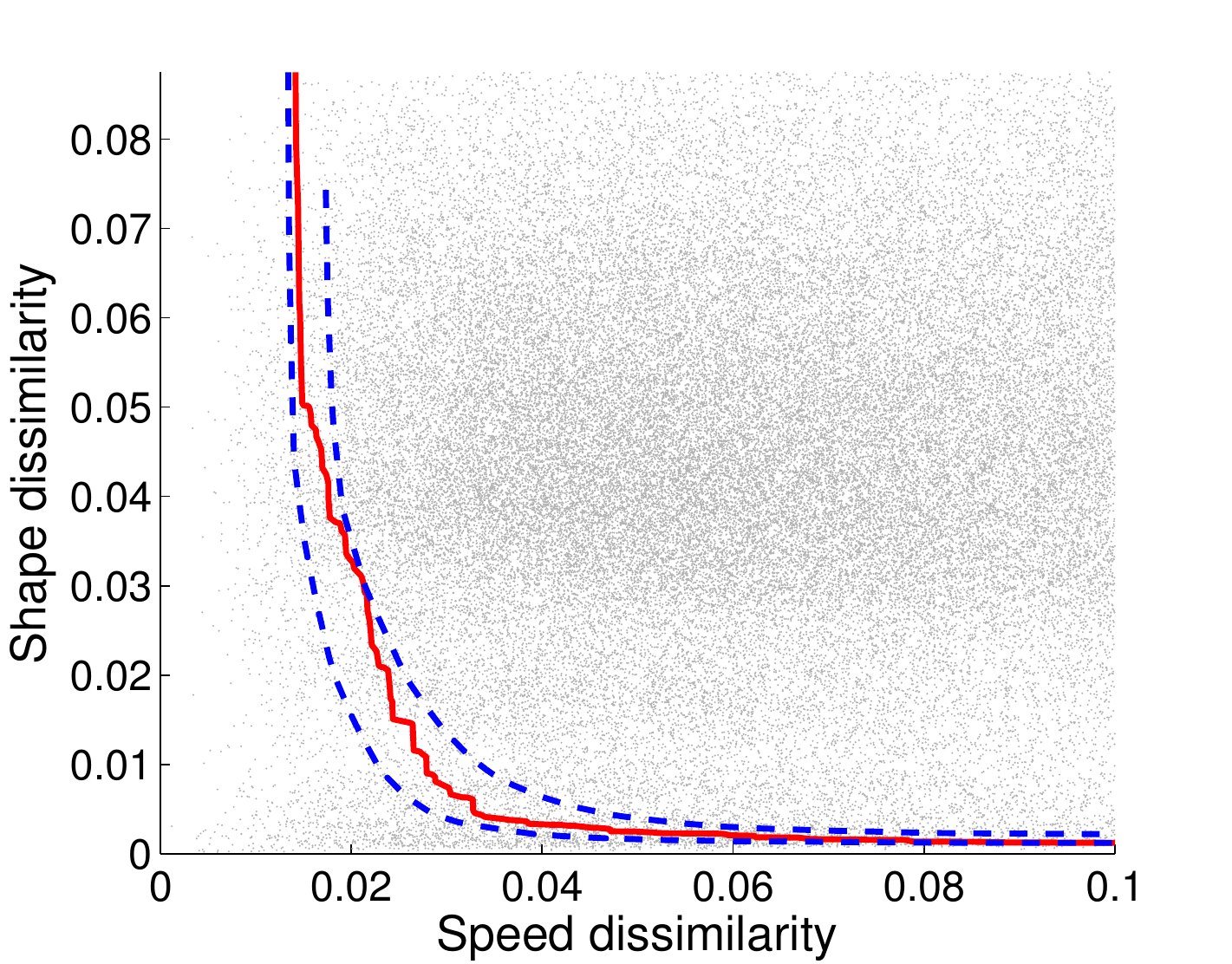}
  \caption{Comparison of a Pareto front (solid red line) on dyads (gray dots) with convex fronts (blue dashed lines) obtained by linear scalarization. The dyads towards the middle of the Pareto front are found in deeper convex fronts than those towards the edges. The result would be inflated anomaly scores for the samples associated with the dyads in the middle of the Pareto fronts when using linear scalarization.}
  \label{convexVSpareto}
  \end{center}
\end{figure}

Many of the Pareto fronts on the dyads are \emph{non-convex}, partially explaining the superior performance of the proposed PDA algorithm. 
The non-convexities in the Pareto fronts lead to inflated anomaly scores for linear scalarization. 
A comparison of a Pareto front with two convex fronts (obtained by scalarization) is shown in Fig.~\ref{convexVSpareto}. 
The two convex fronts denote the shallowest and deepest convex fronts containing dyads on the illustrated Pareto front. 
The test samples associated with dyads near the middle of the Pareto fronts would suffer the aforementioned score inflation, as they would be found in deeper convex fronts than those at the tails.

\begin{figure}[t]
\begin{center}
  \includegraphics[width=7 cm]{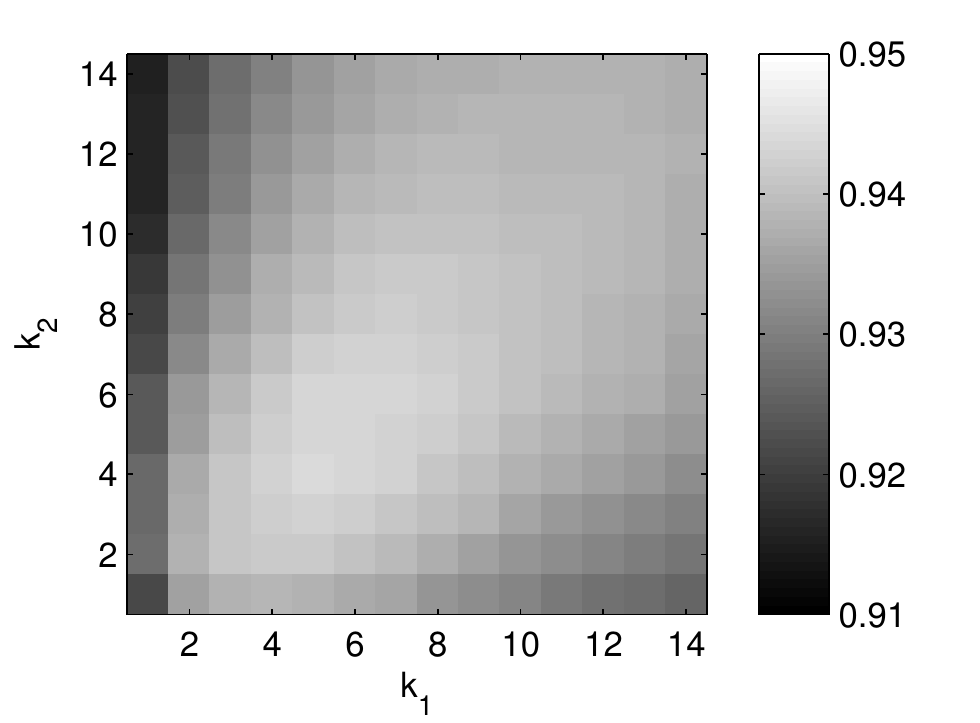}
  \caption{AUCs for different choices of $[k_1,k_2]$ in the pedestrian 
  trajectories experiment. 
  The AUC for the parameters chosen using the proposed heuristic $[k_1=6,k_2=7]$ is within $0.001$ of the maximum AUC obtained by the parameters $[k_1=5,k_2=4]$.}
  \label{auc_heatmap}
  \end{center}
\end{figure}

Finally we note that the proposed PDA algorithm does not appear to be very sensitive to the choices of the number of neighbors, as shown in Fig.~\ref{auc_heatmap}. 
In fact, the heuristic proposed for choosing the $k_l$'s in Section \ref{param} performs quite well in this experiment. Specifically, the AUC obtained when using the parameters chosen by the proposed heuristic is very close to the maximum AUC over all choices of the number of neighbors $[k_1,k_2]$. 

\section{Conclusion}
\label{con}
In this paper we proposed a method for similarity-based anomaly detection using a novel multi-criteria dissimilarity measure, the Pareto depth. 
The proposed method utilizes the notion of Pareto optimality to detect anomalies under multiple criteria by examining the Pareto depths of dyads corresponding to a 
test sample. 
Dyads corresponding to an anomalous sample tended to be located at deeper fronts 
compared to dyads corresponding to a nominal sample. 
Instead of choosing a specific weighting or performing a grid search on the 
weights for dissimilarity measures corresponding to different criteria, the proposed method can 
efficiently detect anomalies in a manner that is tractable for a large 
number of criteria. 
Furthermore, the proposed Pareto depth analysis (PDA) approach is provably better than using linear combinations of criteria. Numerical studies validated our theoretical predictions of PDA's performance advantages compared to using linear combinations on simulated and real data.

An interesting avenue for future work is to extend the PDA approach to extremely 
large data sets using approximate, rather than exact, Pareto fronts. 
In addition to the skyline algorithms from the information retrieval community 
that focus on approximating the first Pareto front, there has been recent work on 
approximating Pareto fronts using partial differential equations \cite{Calder2013} 
that may be applicable to the multi-criteria anomaly detection problem. 

\appendix
The proofs for Theorems \ref{thm:pareto-asymp} and \ref{thm:small} are presented after some preliminary results. 
We begin with a general result on the expectation of $K_n$.  Let $F: [0,1]^d\to \R$ denote the cumulative distribution function of $f$, defined by
\[F(x) = \int_0^{x_1} \cdots \int_0^{x_d} f(y_1,\dots,y_d) \, dy_1 \cdots dy_d.\]
\begin{proposition}\label{prop:num_pts}
For any $n\geq 1$ we have
\begin{equation*}\label{eq:num_pts}
E(K_n) = n \int_{[0,1]^d} f(x) \left(1-F(x)\right)^{n-1} dx.
\end{equation*}
\end{proposition}
\begin{proof}
Let $E_i$ be the event that $Y_i \in \F^n$ and let $\chi_{E_i}$ be indicator random variables for $E_i$.  Then
\[E(K_n) = E\left(\sum_{i=1}^n \chi_{E_i}\right) = \sum_{i=1}^n P(E_i) = nP(E_1).\]
Conditioning on $Y_1$ we obtain
\[E(K_n) =n \int_{[0,1]^d} f(x) P(E_1 \, | \, Y_1=x) dx.\]
Noting that $P(E_1 \, | \, Y_1=x) =\left(1-F(x)\right)^{n-1}$ completes the proof.
\end{proof}

The following simple proposition is essential in the proofs of Theorem \ref{thm:pareto-asymp} and \ref{thm:small}.
\begin{proposition}\label{prop:prelim-asymp}
Let $0 < \delta \leq 1$ and $a > 0$.  For $a \leq \delta^{-d}$ we have
\begin{equation}\label{eq:prelim-asymp-a}
n\int_{[0,\delta]^d} (1-a x_1 \cdots x_d)^{n-1} \, dx = \frac{c_{n,d}}{a} + O((\log n)^{d-2}), 
\end{equation}
and for $a \leq 1$ we have
\begin{equation}\label{eq:prelim-asymp-b}
n \int_{[0,1]^d\setminus [0,\delta]^d} (1-ax_1\cdots x_d)^{n-1} \, dx = O((\log n)^{d-2}).
\end{equation}
\end{proposition}
\begin{proof}
We will give a sketch of the proof as similar results are well-known \citep{bai2005}.  Assume $\delta =1$ and let $Q_n$ denote the quantity on the left hand side of \eqref{eq:prelim-asymp-a}.  Making the change of variables $y_i = x_i$ for $i=1,\dots,d-1$ and $t=x_1\cdots x_d$, we see that
\begin{align*}
Q_n &= n\int_0^1 \int_{t}^1 \int_{\frac{t}{y_{d-1}}}^1\\
&\hspace{2cm}\cdots \int_{\frac{t}{y_2 \cdots y_{d-1}}}^1 \frac{(1-at)^{n-1}}{y_1 \cdots y_{d-1}}  \, dy_1 \cdots dy_{d-1} dt.
\end{align*}
By computing the inner $d-1$ integrals we find that
\[Q_n = \frac{n}{(d-1)!} \int_0^1 (-\log t)^{d-1} (1-at)^{n-1} dt,\]
from which the asymptotics \eqref{eq:prelim-asymp-a} can be easily obtained by another change of variables $u=nat$, provided $a\leq1$.  For $0 < \delta < 1$, we make the change of variables $y=x/\delta$ to find that
\[Q_n = \delta^d  n \int_{[0,1]^d}  (1-a\delta^d y_1\cdots y_d)^{n-1} \, dy.\]
We can now apply the above result provided $a\delta^d \leq 1$.  The asymptotics in \eqref{eq:prelim-asymp-a} show that
\begin{align*}
&n\int_{[0,1]^d} (1-a x_1 \cdots x_d)^{n-1} \, dx \\
&\hspace{1cm}=  n\int_{[0,\delta]^d} (1-a x_1 \cdots x_d)^{n-1} \, dx  + O((\log n)^{d-2}),
\end{align*}
when $a \leq 1$, which gives the second result \eqref{eq:prelim-asymp-b}.
\end{proof}

We now give the proof of Theorem \ref{thm:pareto-asymp}.

\begin{proof}
Let $\eps > 0$ and choose $\delta >0$ such that
\begin{equation*}\label{eq:smoothness-f}
f(0)-\eps \leq f(x) \leq f(0) + \eps  \ \ \text{ for any }  x \in [0,\delta]^d,
\end{equation*}
and $f(0) < \delta^{-d}$.
Since $\sigma \leq f \leq M$, we have that $F(x) \geq \sigma x_1\cdots x_d$ for all $x \in [0,1]^d$.  Since $f$ is a probability density on $[0,1]^d$, we must have $\sigma \leq 1$.  Since $\sigma > 0$, we can apply Proposition \ref{prop:prelim-asymp} to find that
\begin{align}\label{eq:B}
&n \int_{[0,1]^d \setminus [0,\delta]^d} f(x) (1-F(x))^{n-1} \, dx \notag \\
&\hspace{2.5cm}\leq Mn\int_{[0,1]^d\setminus [0,\delta]^d} (1-\sigma x_1\cdots x_d)^{n-1} \, dx  \notag\\
&\hspace{2.5cm}{}={} O((\log n)^{d-2}).
\end{align}
For $x \in [0,\delta]^d$, we have 
\[(f(0) - \eps) x_1\cdots x_d \leq F(x) \leq (f(0)+\eps) x_1\cdots x_d.\]
Combining this with Proposition \ref{prop:prelim-asymp}, and the fact that $f(0) -\eps < \delta^{-d}$ we have
\begin{align}\label{eq:A}
&n \int_{[0,\delta]^d} f(x) (1-F(x))^{n-1} \, dx \notag\\
&\hspace{1cm}\leq (f(0)+\eps) n\int_{[0,\delta]^d} (1-(f(0)-\eps) x_1\cdots x_d)^{n-1} \, dx \notag \\
&\hspace{1cm}{}={} \frac{f(0)+\eps}{f(0)-\eps}\cdot  c_{n,d} + O((\log n)^{d-2}).
\end{align}
Combining \eqref{eq:B} and \eqref{eq:A} with Proposition \eqref{prop:num_pts} we have
\[E(K_n) \leq \frac{f(0)+\eps}{f(0)-\eps}\cdot  c_{n,d} + O((\log n)^{d-2}).\]
It follows that
\[\limsup_{n\to \infty} \, c_{n,d}^{-1}E(K_n) \leq \frac{f(0)+\eps}{f(0)-\eps}.\]
By a similar argument we can obtain
\[\liminf_{n\to \infty} \, c_{n,d}^{-1}E(K_n) \geq \frac{f(0)-\eps}{f(0)+\eps}.\]
Since $\eps>0$ was arbitrary, we see that
\[\lim_{n\to \infty}  \, c_{n,d}^{-1}E(K_n)  = 1. \qedhere\]
\end{proof}

The proof of Theorem \ref{thm:small} is split into the following two lemmas. 
It is well-known, and easy to see, that  $x \in  \L^n$ if and only if $x\in \F^n$,  and $x$ is on the boundary of the convex hull of $\G^n$ \citep{ehrgott2005}.  This fact will be used in the proof of Lemma \ref{lem:upper}.
\begin{lemma}\label{lem:upper}
Assume $f:[0,1]^d \to \R$ is continuous at the origin and there exists $\sigma,M>0$ such that $\sigma \leq f \leq M$. Then 
\begin{equation*}
E(L_n) \leq  \frac{3d-1}{4d-2} \cdot c_{n,d}+ o((\log n)^{d-1}) \ \ {\rm as} \ \ n \to \infty.
\end{equation*}
\end{lemma}
\begin{proof}
Let $\eps > 0$ and choose $0 < \delta<\frac{1}{2}$ so that 
\begin{equation}\label{eq:smoothness-f2}
f(0) - \eps \leq f(x) \leq f(0) + \eps \ \ \text{for any } x \in [0,2\delta]^d,
\end{equation}
and $3f(0) \leq \delta^{-d}$.
As in the proof of Proposition \ref{prop:num_pts} we have $E(L_n) = nP(Y_1 \in \L^n)$, so conditioning on $Y_1$ we have
\[E(L_n) = n \int_{[0,1]^d} f(x) P(Y_1 \in \L^n \, | \, Y_1=x) \, dx.\]
As in the proof of Theorem \ref{thm:pareto-asymp}, we have
\begin{align*}
&n\int_{[0,1]^d\setminus [0,\delta]^d} f(x)P(Y_1 \in \L^n \, | \, Y_1 = x) \, dx \\
&\hspace{2.5cm}\leq n\int_{[0,1]^d \setminus [0,\delta]^d} f(x) (1-F(x))^{n-1} \, dx  \\
&\hspace{2.5cm}{}={} O((\log n)^{d-2}),
\end{align*}
and hence
\begin{align}\label{eq:exp-L}
E(L_n) &= n\int_{[0,\delta]^d}f(x) P(Y_1 \in \L^n \, | \, Y_1 = x)\, dx \notag \\
&\hspace{4cm}+ O((\log n)^{d-2}).
\end{align}
Fix $x \in [0,\delta]^d$ and define $A = \{y \in [0,1]^d \, : \, \forall i, \ y_i \leq x_i\}$
and
\begin{align*}
B_i &= \bigg\{y \in [0,1]^d \ \ : \ \ \forall j\neq i, \ y_j < x_j \\
&\hspace{4cm} \text{and} \ \  x_i < y_i < 2x_i - \frac{x_i}{x_j} y_j\bigg\}, 
\end{align*}
for $i=1,\dots,d$, and note that $B_i  \subset [0,2\delta]^d$ for all $i$. 
\begin{figure}[!t]
\centering
\includegraphics[width=0.4\textwidth]{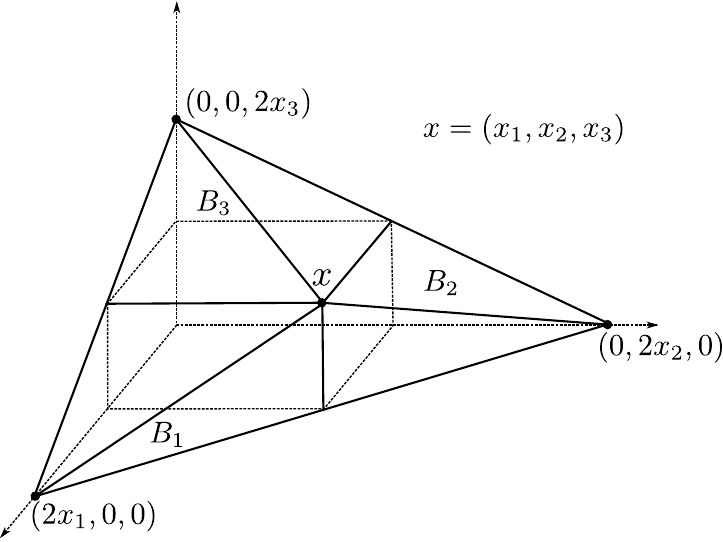}
\caption{Depiction of the sets $B_1,B_2$ and $B_3$ from the proof of Lemma \ref{lem:upper} in the case that $d=3$.}
\label{fig:demo}
\end{figure}
See Fig.~\ref{fig:demo} for an illustration of these sets for $d=3$. 

We claim that if at least two of $B_1,\dots,B_d$ contain samples from $Y_2,\dots,Y_n$, and $Y_1=x$, then $Y_1 \not\in \L^n$.    
To see this, assume without loss of generality that $B_1$ and $B_2$ are nonempty and let $y \in B_1$ and $z \in B_2$. Set 
\[\tilde{y} = \left(y_1, 2x_2 - \frac{x_2}{x_1} y_1, x_3,\dots,x_d \right)\]
\[\tilde{z} = \left(2x_1 - \frac{x_1}{x_2}z_2,z_2,x_3,\dots,x_d \right).\]
By the definitions of $B_1$ and $B_2$ we see that $y_i \leq \tilde{y}_i$ and $z_i \leq \tilde{z}_i$ for all $i$, hence $\tilde{y},\tilde{z} \in \G_n$.  Let $\alpha \in (0,1)$ such that 
\[\alpha y_1 + (1-\alpha) \left(2x_1 - \frac{x_1}{x_2} z_2\right) = x_1.\]
A short calculation shows that $x = \alpha \tilde{y} + (1-\alpha)\tilde{z}$ which implies that $x$ is in the interior of the convex hull of $\G_n$, proving the claim.  

Let $E$ denote the event that at most one of $B_1,\dots,B_d$ contains a sample from $Y_2,\dots,Y_n$, and let $F$ denote the event that $A$ contains no samples from $Y_2,\dots,Y_n$.   Then by the observation above we have
\begin{equation}\label{eq:bounded-event}
P(Y_1 \in \L^n \, | \, Y_1=x) \leq P(E\cap F \, | \, Y_1=x) = P(E\cap F).
\end{equation}
For $i=1,\dots,d$, let $E_i$ denote the event that $B_i$ contains no samples from $Y_2,\dots,Y_n$.  It is not hard to see that 
\[E = \bigcup_{i=1}^d \left(\bigcap_{j\neq i} E_j \setminus \bigcap_j E_j\right) \bigcup \left(\bigcap_j E_j\right).\]
Furthermore, the events in the unions above are mutually exclusive (disjoint) and $\cap_j E_j \subset \cap_{j\neq i} E_j$ for $i=1,\dots,d$. It follows that
\begin{align}\label{eq:temp}
&P(E\cap F)\notag \\
&{}={} \sum_{i=1}^d \left(P\left(\cap_{j\neq i} E_j\cap F\right) - P\left(\cap_j E_j \cap F\right) \right) + P\left(\cap_j E_j\cap F\right) \notag \\
&{}={} \sum_{i=1}^d P\left(\cap_{j\neq i} E_j\cap F\right)  - (d-1) P\left(\cap_j E_j\cap F\right) \notag \\
&{}={} \sum_{i=1}^d \left(1-F(x) - \int_{\cup_{j\neq i} B_j}\hspace{-7mm} f(y) \, dy\right)^{n-1}\notag \\
&\hspace{2cm} - (d-1) \left(1-F(x) - \int_{\cup_j B_j}\hspace{-3mm} f(y) \, dy\right)^{n-1}\hspace{-5mm}.
\end{align}
A simple computation shows that $|B_j| =\frac{1}{d} x_1\cdots x_d$ for $j=1,\dots,d$.  Since $A,B_i \subset [0,2\delta]^d$, we have by \eqref{eq:smoothness-f2} that 
\[(f(0) - \eps)x_1\cdots x_d \leq F(x) \leq (f(0)+\eps) x_1\cdots x_d,\]
and
\[\frac{1}{d} (f(0) - \eps) x_1\cdots x_d \leq \int_{B_j} f(y) \, dy \leq \frac{1}{d} (f(0)+\eps) x_1\cdots x_d.\]
Inserting these into \eqref{eq:temp} and combining with \eqref{eq:bounded-event} we have
\begin{align*}
&P(Y_1 \in \L^n \, | \, Y_1=x) \\
&\hspace{1cm}\leq d\left(1 - \frac{2d-1}{d}(f(0)-\eps) x_1\cdots x_d\right)^{n-1}  \\
&\hspace{2.5cm} - (d-1) \left( 1- 2(f(0)+\eps) x_1\cdots x_d\right)^{n-1}.
\end{align*}
We can now insert this into \eqref{eq:exp-L} and apply Proposition \ref{prop:prelim-asymp} (since $3f(0)\leq \delta^{-d}$) to obtain
\begin{align*}
E(L_n) &\leq \left(\frac{d^2}{2d-1} \frac{f(0)+\eps}{f(0) - \eps} - \frac{d-1}{2} \frac{f(0)-\eps}{f(0)+\eps} \right) c_{n,d}\\
&\hspace{5cm} + O((\log n)^{d-2}).
\end{align*}
Since $\eps >0$ was arbitrary, we find that
\[\limsup_{n \to \infty} \  c_{n,d}^{-1} E(L_n) \leq \left(\frac{d^2}{2d-1}  - \frac{d-1}{2}  \right) = \frac{3d-1}{4d-2}.\qedhere\] 
\end{proof}
\begin{lemma}\label{lem:lower}
Assume $f:[0,1]^d \to \R$ is continuous and there exists $\sigma,M>0$ such that $\sigma \leq f \leq M$.  Then 
\begin{equation*}
E(L_n) \geq  \frac{d!}{d^d} \cdot c_{n,d} + o((\log n)^{d-1}) \ \ {\rm as} \ \ n \to \infty.
\end{equation*}
\end{lemma}
\begin{proof}
Let $\eps > 0$ and choose $0 < \delta < 1/d$ so that 
\begin{equation}\label{eq:smoothness-f3}
f(0) - \eps \leq f(x) \leq f(0) + \eps \ \ {\rm for} \ \ x \in [0,d\delta]^d,
\end{equation}
and
\begin{equation}\label{eq:bound}
\frac{d^d}{d!} (f(0) + \eps) \leq \delta^{-d}.
\end{equation}
As in the proof of Lemma \ref{lem:upper} we have
\begin{align}\label{eq:exp-L2}
E(L_n) &{}={} n\int_{[0,\delta]^d}f(x) P(Y_1 \in \L^n \, | \, Y_1 = x)\, dx \notag \\
&\hspace{4cm}+ O((\log n)^{d-2}).
\end{align}
Fix $x \in (0,\delta)^d$, set $\nu = \left(\frac{1}{x_1},\dots,\frac{1}{x_d}\right)$ and 
\[A = \big\{ y \in [0,1]^d \ \ | \ \ y\cdot \nu \leq x \cdot \nu\big\}.\]
Note that $A$ is a simplex with an orthogonal corner at the origin and side lengths $d\cdot x_1,\dots,d\cdot x_d$. A simple computation shows that $|A| = \frac{d^d}{d!} x_1\cdots x_d$.  By \eqref{eq:smoothness-f3} we have
\[\int_A f(y) \, dy \leq  (f(0)+\eps) |A| = \frac{d^d}{d!}(f(0)+\eps)x_1\cdots x_d.\] 
It is easy to see that if $A$ is empty and $Y_1=x$ then $Y_1 \in \L^n$, hence
\begin{align*}
P(Y_1 \in \L^n \, | \, Y_1=x) &{}\geq{} \left( 1- \int_{A} f(y)\, dy\right)^{n-1} \\
&{}\geq{} \left( 1 - \frac{d^d}{d!} (f(0)+\eps) x_1\cdots x_d\right)^{n-1}.
\end{align*}
Inserting this into \eqref{eq:exp-L2} and noting \eqref{eq:bound}, we can apply Proposition \ref{prop:prelim-asymp} to obtain
\[E(L_n) \geq \frac{d!}{d^d}\frac{f(0)-\eps}{f(0)+\eps}c_{n,d}  + O((\log n)^{d-2}),\]
and hence
\[\limsup_{n\to \infty} \ c_{n,d}^{-1} E(L_n) \geq \frac{d!}{d^d}. \qedhere\]
\end{proof}

Theorem \ref{thm:small} is obtained by combining Lemmas \ref{lem:upper} and 
\ref{lem:lower}.

\bibliographystyle{IEEEtran}
\bibliography{ref}

\begin{IEEEbiography} 
[{\includegraphics[width=1in]{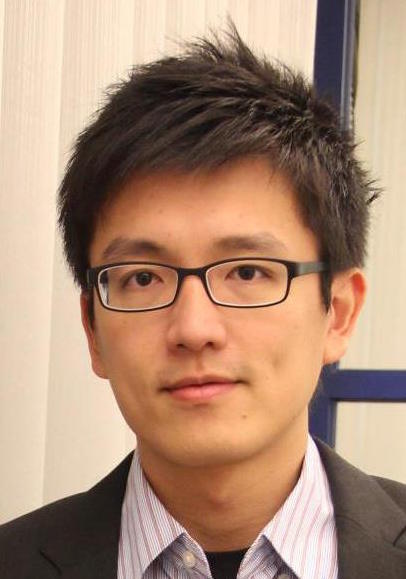}}]
{Ko-Jen Hsiao} obtained his Ph.D. in Electrical Engineering and Computer Science at the University of Michigan in 2014. He received the B.A.Sc.~degree in Electrical Engineering from the National Taiwan University in 2008 and the Master degrees in Electrical Engineering: Systems and Applied Mathematics from the University of Michigan in 2011. Ko-Jen Hsiao's main research interests are in statistical signal processing and machine learning with applications to anomaly detection, image retrieval, recommendation systems and time series analysis. He is currently a data scientist and an applied researcher in machine learning at WhisperText.
\end{IEEEbiography}

\begin{IEEEbiography} 
[{\includegraphics[width=1in]{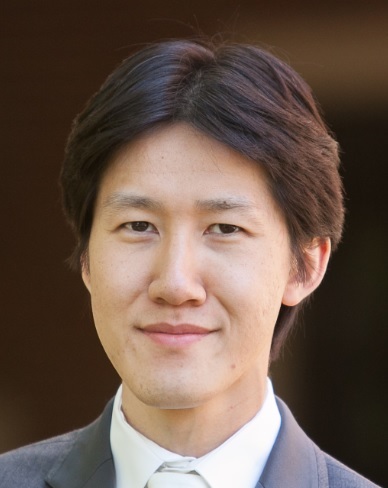}}]
{Kevin S. Xu} received the B.A.Sc. degree in Electrical Engineering from the University of Waterloo in 2007 and the M.S.E. and Ph.D. degrees in Electrical Engineering: Systems from the University of Mich-igan in 2009 and 2012, respectively. He was a recipient of the Natural Sciences and Engineering Re-search Council of Canada (NSERC) Postgraduate MasterÕs and Doctorate Scholarships. He is currently an assistant professor in the EECS Department at the University of Toledo and has previously held in-dustry research positions at Technicolor and 3M. His main research interests are in machine learning and statistical signal processing with applications to network science and human dynamics.
\end{IEEEbiography}

\begin{IEEEbiography} 
[{\includegraphics[width=1in]{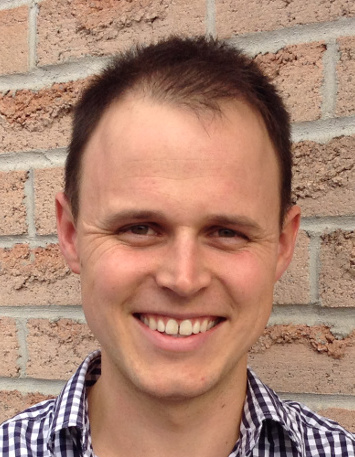}}]
{Jeff Calder} is currently a Morrey Assistant Professor of Mathematics with the University of California at Berkeley, Berkeley, CA, USA. He received the Ph.D. degree in applied and interdisciplinary mathematics from the University of Michigan, Ann Arbor, MI, USA, in 2014 under the supervision of Selim Esedoglu and Alfred Hero. His primary research interests are the analysis of partial differential equations, applied probability, and mathematical problems in computer vision and image processing. He received the M.Sc. degree in mathematics from QueenÕs University, Kingston, ON, Canada, under the supervision of A.-R. Mansouri, where he developed Sobolev gradient flow techniques for image diffusion and sharpening, and the B.Sc. degree in mathematics and engineering from QueenÕs University with a specialization in control and communications.

\end{IEEEbiography}

\begin{IEEEbiography}
[{\includegraphics[width=1in,height=1.25in]{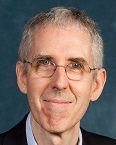}}]
{Alfred O.~Hero III} received the B.S. (summa cum laude) from Boston University (1980) and the Ph.D from Princeton University (1984), both in Electrical Engineering. Since 1984 he has been with the University of Michigan, Ann Arbor, where he is the R. Jamison and Betty Williams Professor of Engineering and co-director of the Michigan Institute for Data Science (MIDAS) . His primary appointment is in the Department of Electrical Engineering and Computer Science and he also has appointments, by courtesy, in the Department of Biomedical Engineering and the Department of Statistics. From 2008-2013 he held the Digiteo Chaire d'Excellence at the Ecole Superieure d'Electricite, Gif-sur-Yvette, France. He is a Fellow of the Institute of Electrical and Electronics Engineers (IEEE) and several of his research articles have received best paper awards. Alfred Hero was awarded the University of Michigan Distinguished Faculty Achievement Award (2011). He received the IEEE Signal Processing Society Meritorious Service Award (1998), the IEEE Third Millenium Medal (2000), and the IEEE Signal Processing Society Technical Achievement Award (2014). Alfred Hero was President of the IEEE Signal Processing Society (2006-2008) and was on the Board of Directors of the IEEE (2009-2011) where he served as Director of Division IX (Signals and Applications). He served on the IEEE TAB Nominations and Appointments Committee (2012-2014). Alfred Hero is currently a member of the Big Data Special Interest Group (SIG) of the IEEE Signal Processing Society. Since 2011 he has been a member of the Committee on Applied and Theoretical Statistics (CATS) of the US National Academies of Science.
Alfred Hero's recent research interests are in the data science of high dimensional spatio-temporal data, statistical signal processing, and machine learning. Of particular interest are applications to networks, including social networks, multi-modal sensing and tracking, database indexing and retrieval, imaging, biomedical signal processing, and biomolecular signal processing.
\end{IEEEbiography}

\end{document}